\pdfoutput=1
\documentclass[10pt,letterpaper]{article}
\usepackage[top=0.85in,footskip=0.75in,marginparwidth=2in]{geometry}
\usepackage{amsmath,amssymb}
\usepackage{enumerate}
\usepackage{multicol}
\usepackage{multirow}
\usepackage{booktabs} 
\usepackage{tabu}
\usepackage{hyperref}
\usepackage{graphicx}
\usepackage{authblk}

\newcommand{\E}{\mathbb{E}}

\newcommand{\R}{\mathbb{R}}

\newcommand{\Ac}{\mathcal{A}}
\newcommand{\Dc}{\mathcal{D}}

\newcommand{\Fc}{\mathcal{F}}
\newcommand{\Kc}{\mathcal{K}}
\newcommand{\Lc}{\mathcal{L}}

\newcommand{\Xc}{\mathcal{X}}
\newcommand{\Yc}{\mathcal{Y}}

\newcommand{\fsd}{\succeq_{(1)}}
\newcommand{\sfsd}{\succ_{(1)}}

\newcommand{\dist}{\mathop{\rm dist}}

\newtheorem{theorem}{Theorem}
\newtheorem{proof}{Proof}
\newtheorem{lemma}[theorem]{Lemma}

\newtheorem{corollary}[theorem]{Corollary}
\newtheorem{definition}[theorem]{Definition}

\begin{document}

\title{Risk--Averse Classification}
\author[1]{Constantine Alexander Vitt}
\author[2]{Darinka Dentcheva}
\author[1]{Hui Xiong}
\affil[1]{Rutgers University, Newark and New Brunswick, NJ}
\affil[2]{Stevens Insitute of Technology, Hoboken, NJ}

\maketitle


\begin{abstract}
We develop a new approach to solving classification problems, which
is bases on the theory of coherent measures of risk and risk sharing ideas. The proposed approach aims at designing a risk-averse classifier. The new approach allows for associating distinct risk functional to each classes. The risk may be measured by different (non-linear in probability) measures,

We analyze the structure of the new classifier design problem and establish its theoretical relation to known risk-neutral design problems. In particular, we show that the risk-sharing classification problem is equivalent to an implicitly defined optimization problem with unequal, implicitly defined but unknown, weights for each data point. We implement our methodology in a binary classification scenario on several different data sets and carry out numerical comparison with classifiers which are obtained using the Huber loss function and other loss functions known in the literature. 
We formulate specific risk-averse support vector machines in order to demonstrate the viability of our method.
\end{abstract}


\section{Introduction} 
\label{sec:introduction}

Classification is one of the fundamental tasks of the data mining and machine learning community.
The need for accurate and effectively solution of classification problems proliferates throughout the business world, engineering, and sciences.
In this paper, we propose a new approach to classification problems with the aim to develop a methodology for reliable risk-averse classifiers design which has the flexibly to allow customers choice of risk measurement for the misclassification errors in various classes.
The proposed approach has its foundation on the theory of coherent measures of risk and risk sharing. Although, this theory is well advanced in the field of mathematical finance and actuarial analysis, the classification problem does not fit the problem setting analyzed in those fields and the theoretical results on risk sharing are inapplicable here. The classification problem raises new issues, poses new challenges, and requires a dedicated analysis.

We consider labeled data consisting of $k$ subsets $S_1,\dots, S_k$ of $n$-dimensional vectors. The cardinality of $S_i$ is $|S_i|=m_i$, $i=1,\dots, k$.
Analytically, the classification problem consists of identifying a mapping $\phi$, whose image can be partitioned into $k$ subsets corresponding to each class of data, so that $\phi (\cdot)$ can be used as an indicator function of each class. We adopt the following definition.
\begin{definition}
\label{def:general}
A classifier is a vector function $\varphi:\R^n\to\R^d$ such that $\varphi(x) \in K_i$ for all $x\in S_i$, $i=1,\dots,k$, where $K_i\subset\R^d$ and $K_i\cap K_j=\emptyset$ for all $i,j=1,\dots,k$ and $i\neq j$.
\end{definition}
In our discussion, we assume that the classifier belongs to a certain functional family depending on a finite number of parameters, which we denote by $\pi\in\R^s$. The task is to choose a suitable values for the parameter $\pi$. An example of this point of view is given by the support vector machine, in which $k=2$, $\varphi(x;\pi) = v^\top x -\gamma$, $\pi= (v,\gamma)\in\R^{n+1}$, and $K_1= (0,+\infty)$, $K_2= (-\infty,0)$.

Some examples of this point of view are the following.
When support vector machine is formulated, we seek to distinguish two classes, i.e., $k=2$. The classifier is a linear function $\varphi(x;\pi) :\mathbb R^n\to \mathbb R$, defined by setting 
\[
\varphi(x;\pi)= v^\top x -\gamma \text{ for any } x\in\mathbb R^n. 
\]
The classifier is determined by the parameters $\pi= (v,\gamma)\in\R^{n+1}$. The regions the classifier maps to are $K_1= [0,+\infty)$, $K_2= (-\infty,0)$. 

Let us consider the case of separating many classes, e.g., $k\geq 3$ by the creating a linear classifier on the principle ``one vs. all''. 
Then effectively,  our goal is to determine functions $\varphi_j(x;a^j,b_j):=\langle a^j, x \rangle - b_j$, where $x$ is a data point from the feature space, $a^j\in\mathbb R^n$, $j=1,\dots k-1$, are the normals of the separating planes and $b_j$ determine the location of the $j$-th plane. Plane $j$ is meant to separate the data points from class $j$ from the rest of the data points. This means that
\begin{equation}
\label{e:mult-class-def}
\varphi_j(x;a^j,b_j) =\begin{cases}\geq 0 & \text{ for } x\in S_j\\
 < 0 & \text{ for } x\not\in S_j.
\end{cases}
\end{equation}
We define a $k-1\times n$ matrix $A$ whose rows are the vectors $a^j$, and a vector $b\in\mathbb R^{k-1}$ whose components are $b_j$. The classifier for this problem can be viewed as a vector function $\varphi(\cdot; A,b): \mathbb R^n\to \mathbb R^{k-1}$ by setting $\varphi (x; A,b) = Ax-b$. 
The parameter space is of form $\pi=(A,b)\in\mathbb R^{(k-1)(n+1)}$. 
Requirement \eqref{e:mult-class-def} means that
the regions $K_j$ are the orthants 
\begin{align*}
K_i & = \{z\in\mathbb R^{k-1}: \, z_i\geq 0, \, z_j < 0,\;\; j\ne i, \;j =1,\dots, k-1\}, \; i=1,\dots k-1;\\
K_k & = \{z\in\mathbb R^{k-1}: \, z_i < 0, \; i =1,\dots, k-1\}
\end{align*}
This setting may be used for classification in the anomaly detection scenario. Two approaches are possible. One setting may require to distinguish between several distinct normal regimes or features of normal operational status. In that case, the class $k$ may contain the anomalous 
instances, while classes $i=1,\dots k-1$ represent the normal operation. Another problem deals with several rare undesirable phenomena with distinct features.
In such a scenario, we may associate classes $i=1,\dots k-1$ with those anomalous events and class $k$ with a normal operation. 
When kernels are used, then the mapping $\varphi(x;\pi)$ becomes a composition of a projection mapping to the reduced feature space and a classifier mapping in the feature space.


\section{Loss Functions} 
\label{sec:loss_functions}

A key element, which distinguishes various classification approaches, is the choice of a loss function, which, typically, is one of the known risk functionals in statistical model fitting. The quality of every model is determined by analysis of the residuals, e.g. the error. Let us introduce the following notation. For a random observation $z\in\R^n$, we calculate $\varphi(z;\pi)$ and note that misclassification occurs when $\varphi(z;\pi)\not\in K_i$, while $z\in S_i$ for any $i=1,\dots, k$.
In statistical terms, we try to predict the membership $y\in\{1,\dots, k\}$ of a data point to one of the classes. 
The classification error can be defined as the distance of a particular record to the classification set, to which it should belong. 
Here the distance from a point $r$ to a set $K$ is defined by using a suitable norm in $\R^n$:
\[
\dist(r,K)  =\min\{ \|r-a\|: a\in K\}.
\]
Note that here we assumes implicitly that the set $K$ is convex and closed. Indeed, the sets $K_i$, $i=1\dots,k$ are closed convex set for most classification problems, as evidenced by the examples in the previous section.

As the records in every data class $S_i$, $i=1,\dots, k$ constitute a sample of an unknown distribution of a random vector $X^i$ defined on a probability space $(\varOmega,\Fc,P)$, we define the following random variables:
\begin{equation}
\label{d:error}
Z^i(\pi)=\dist(\varphi(X^i;\pi),{K_i}),\,\, i=1,\dots k,
\end{equation}
represent the misclassification of records in class $i$ when parameter $\pi$ is used. 
These are univariate random variables defined on the same probability space and are represented by the sampled observations
\[
Z^i_j(\pi)=\dist(\varphi(x_j;\pi),{K_i}) \text{ with } x_j\in S_i\quad j=1,\dots, m_i.
\]
The expected misclassification error for each class can be estimated as follows:
\[
\hat{Z}^i (\pi) = \sum_{x_j\in S_i} \frac{1}{m_i}\dist(\varphi(x_j;\pi),K_i)
\]
The following figure illustrates how the classification error for a certain binary classifier is measured.
\begin{figure}[!h]
\begin{center}
\vspace{-2cm}
\includegraphics[width=0.75\linewidth]{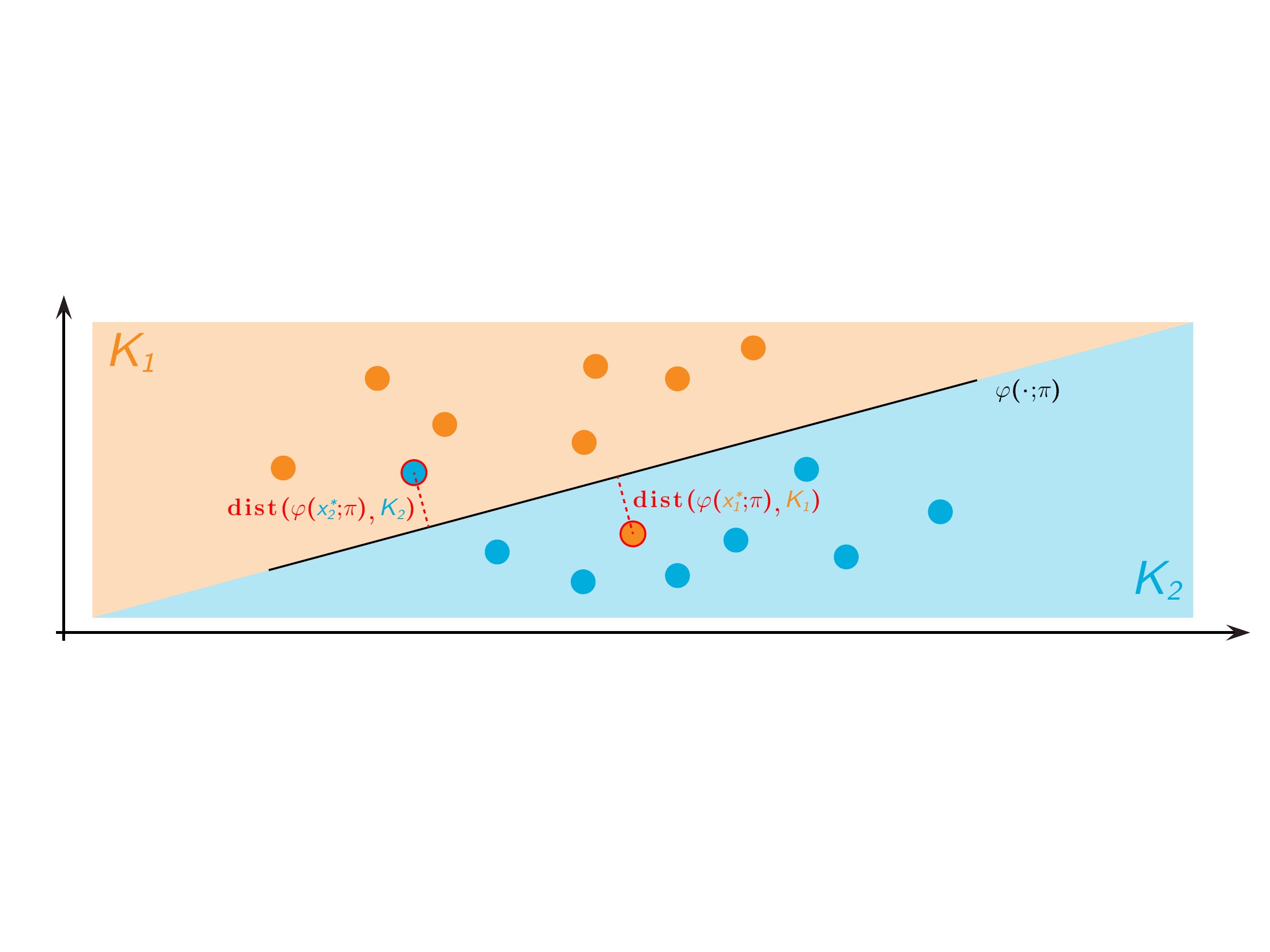}
\label{fig:error_dist}
\vspace{-2cm}
\caption{Classification error calculation}
\end{center}
\end{figure}
In the support vector machine the classification error is computed by 
\[
\dist \big(\varphi(x;v,\gamma), K_i\big)= \begin{cases} \max(0, \langle v, x \rangle -\gamma) & \text{ for } x\in S_1,\\
\max(0, \gamma-\langle v, x\rangle) & \text{ for } x\in S_2.
\end{cases}
\]
We classify every new observation $x$ in $S_i$, if $\dist \big(\varphi(x;v,\gamma), K_i\big) =0$, $i,=1,2$. 
In the case of SVM, the regions cover the entire image space of the classifier $R=K_1\cup K_2.$ Therefore, the condition
$\dist \big(\varphi(x;v,\gamma), K_i\big) =0$, $i,=1,2$,  always holds for exactly one class.

Observe that in the multi-class example, the regions $K_i$, $i=1,\dots k$ do not cover the entire image space of the classifier. Therefore, it is possible to observe a future instance $x$ such that $\dist \big(\varphi(x; A,b), K_i\big)>0$ for all $i=1,\dots k$. 
In that case,  we could classify according to the smallest distance 
\[
x\in S_j\quad \text{iff} \quad \dist \big(\varphi(x; A,b), K_j\big)= \min_{1\leq i\leq k} \dist \big(\varphi(x; A,b), K_i\big),\quad j\in\{1,\dots,k\}.
\]
Another problem arises, if the the minimum distance is achieved for several classes. The ambiguity could be resolved in several ways as a sequential classification procedure but this question is beyond the scope of our study.


\section{Robust Classification Design and Robust Statistics}

The design of robust estimators, robust classifiers in particular, has attracted attention of statisticians as well as of data scientists. Additionally, the distributions of the populations providing the currently available records may not be well represented by the current sample (e.g., it might have heavy tails, not be unimodal, etc.)
Furthermore, misclassification may lead to different cost with different probability depending on the error. An example for such a case is the damage caused by a hurricane. If we fail to predict correctly that a hurricane will take place in certain region, the cost of the damage depends on the features used for classification  and is highly non-linear with respect to those features (see \cite{davis2016analysis}).

We refer to \cite{huber2011robust,el2003robust,gotoh2017support,hastie2009elements} and the references therein for methods of robust classification design.
Most cases address binary classification.

Support vector machines are widely used and most popular classification tools. They appear also as part of sequential classification methods for multiple classes. Various approaches in the literature address the design of a robust classifier specifically for the support vector machine. 

We start with the formulation of an optimization problem based on the loss function expressing the minimization of the (estimated) expected total classification error.
The design of a binary classifier can be accomplished by solving the following optimization problem:
\begin{equation}
\label{eq:svmE1}
\begin{aligned}
\min_{v,\gamma,Z^1,Z^2}\, &\,  \frac{1}{m_1} \sum_{j=1}^{m_1} z_j^1 + \frac{1}{m_2} \sum_{j=1}^{m_2} z_j^2  \\
\text{s. t. }\, & \, \langle v, x_j^1\rangle - \gamma + z_j^1 \ge 0, \quad j=1,\dots,m_1,\\
&\, \langle v, x_{j}^2\rangle -\gamma  - z_{j}^2 \le 0, \quad j=1,\dots,m_2,\\
&\, \|v\|=1,\\
&\, Z^1\ge 0, Z^2\geq 0.
\end{aligned}
\end{equation}
In this formulation, $Z^1$ and $Z^2$ are random variables expressing the magnitude of the classification error for class 1 and class 2, respectively. 
Those variables have realizations $z_i^1$ and $z_i^2$. The parameters of the classifier are
$\pi = (v,\gamma)$. Note that proper calculation of the magnitude of classification error requires the use 
of the Euclidean norm of $v$. In that case,  $\langle v, x\rangle =\gamma$ is the equation of a plane and the value $\varphi(x;\pi)=\langle v, x\rangle -\gamma$ is
indicative of the position of the point $x$ relative to that plane: the sign of $\varphi(x;\pi)$ indicates on which side of the plane the point is located and the absolute value of $\varphi(x;\pi)$ indicates how far is the point from the plane. 

The problem is frequently replaced by the so-called soft-margin SVM with parameters $M>0$ and $\delta>0$, formulated as follows:
\begin{equation}
\label{eq:svmE}
\begin{aligned}
\min_{v,\gamma,Z^1,Z^2}\, &\,  M \left( \sum_{j=1}^{m_1} z_j^1 + \sum_{j=1}^{m_2} z_j^2 \right)+\delta\|v\|^2  \\
\text{s. t. }\, & \, \langle v, x_j^1\rangle - \gamma + z_j^1 \ge 1, \quad j=1,\dots,m_1,\\
&\, \langle v, x_{j}^2\rangle -\gamma  - z_{j}^2 \le -1, \quad j=1,\dots,m_2,\\
&\, Z^1\ge 0, Z^2\geq 0.
\end{aligned}
\end{equation} 
In problem \eqref{eq:svmE}, the normal vector $v$ can be of any positive length. Observe that multiplying the solution of problem
\eqref{eq:svmE}, $v$ and $\gamma$, by a positive constant does not change the separating plane. 
In problem \eqref{eq:svmE}, the estimated expected total classification error equals
\[
 \frac{1}{m_1\|v\|} \sum_{i=1}^{m_1} \max (z_i^1-1, 0) + \frac{1}{m_2\|v\|} \sum_{j=1}^{m_2} \max (z_j^2-1, 0)
\]
This means that the objective function does not necessarily minimize the expected classification error although the variables $z^1_j$ and $z^2_j$ are indicative of misclassification occurrence. Therefore, it only makes sense to compare the quality of normalized classifiers, where the length of $v$ is one.

Most notable approach to robust binary classification is provided by the theory and methods of robust statistics. In this approach, the model is fit using the Huber risk function, which is defined for $z\in S_i$, $i=1,2$ as follows:
\begin{equation}
\label{Huberrf}
L_H(z; v,\gamma) = \begin{cases}
\big[\max \big(0, 1+(-1)^i(\gamma-\langle v, z\rangle) \big)\big]^2 &\!\!\text{if} \\
\hfill (-1)^i(\gamma-\langle v, z\rangle) \geq -1\\
(-1)^i(\langle v, z\rangle - \gamma) \qquad \text{otherwise.}
\end{cases}
\end{equation}
Another approach is presented in \cite{lanckriet2003robust,el2003robust}, where the tools of robust optimization are employed. The idea there is that the future instance will come a distribution, which is close to the observed empirical distribution in some sense. Therefore, some set of distrributions is constructed, called uncertainty set, and the minimization is carried out over all distributions in that set. In \cite{lanckriet2003robust,el2003robust}, the uncertainty sets are defined by allowing all distributions on the smape space, which have the same mean and the same covariance as the estimated empirical mean and covariance. 
In \cite{ma2011robust} the authors look at the median hinge loss determined for each class and minimize the sum of the two median losses.

Our proposed approach suggests to minimize the classification error in a risk averse manner.
For this purpose, we propose new family of loss functions, which use coherent measures of risk.

\section{Coherent Measures of Risk} 
\label{sec:coherent_measure_of_risk}

Measures of risk are widely used in finance and insurance. Additionally, the signal to noise measures, used in engineering and statistics (Fano factor \cite{fano1947ionization} or the index of dispersion \cite{cox1966statistical}) are of similar spirit.

An axiomatic theory of measures of risk is presented in \cite{ogryczak1999stochastic,artzner1999coherent,follmer2011stochastic,kijima1993mean,rockafellar2006generalized}
In a more general setting risk measures are analyzed in \cite{ruszczynski2006optimization}. For $p\in [1,\infty]$
and a probability space $(\varOmega,{\Fc},P)$, we use the notation ${\Lc}_p(\varOmega,{\Fc},P)$, for the space of random variables with finite $p$-th moments. We use $\Lc_p(\varOmega)$ for short whenever no ambiguity arises.

\begin{definition} A {\em coherent measure of risk} is a functional $\varrho: \Lc_p(\varOmega)\to \R$ satisfying the following axioms:
\begin{description}
\item[Convexity:] For all $X,Y$, $\gamma\in [0,1]$, $\varrho(\gamma X + (1-\gamma)Y) \le \gamma \varrho(X) + (1-\gamma)\varrho(Y).$
\item[Monotonicity:] If  $X_\omega\ge  Y_\omega$ for $P$-a.a $\omega\in\varOmega$, then $\varrho(X)\ge \varrho(Y)$.
\item[Translation Equivariance:] For any $a\in \R$,  $\varrho(X+a)=\varrho(X)+ a$ for all $X$.
\item[Positive Homogeneity:] If $t>0$ then $\varrho(t X)=t\varrho(X)$ for any $X$.
\end{description}
\end{definition}
For an overview of the theory of coherent measures of risk, we refer to  \cite{shapiro2014lectures} and the references therein.

A risk measure $\varrho(\cdot)$  is called \emph{law-invariant} if $\varrho(X)=\varrho(Y)$ whenever the random variables $X$ and $Y$ have the same distributions. It is clear that in our context, only law invariant measures of risk are relevant.

The following result is know as a dual representation of coherent measures of risk (cf.\cite{shapiro2014lectures}).
The space
${\Lc}_p(\varOmega)$ and the space ${\Lc}_q(\varOmega)$ with $\frac{1}{p} + \frac{1}{q} =1 $ are viewed as paired vector spaces with respect to the bilinear form
\begin{equation}
\label{scprod-a}
\langle \zeta,Z \rangle =\int_\varOmega \zeta (\omega) Z(\omega) dP(\omega),\;
\zeta\in \Lc_q(\varOmega),\;Z\in \Lc_p(\varOmega).
\end{equation}
For any $\zeta\in \Lc_p(\varOmega)$, we can view $\langle \zeta, Z\rangle$ as the expectation $\E_Q[Z]$ taken
with respect to the probability measure $dQ=\zeta dP$, defined by the density $\zeta$, i.e., $Q$ is absolutely  continuous with respect to $P$ and its Radon-Nikodym derivative is $dQ/dP=\zeta$.
For any finite-valued coherent measure of risk  $\rho$, a convex subset $\Ac$  of probability density functions $\zeta\in \Lc_q(\varOmega)$ exists, such that for any random variable $Z\in {\Lc}_p(\varOmega)$, it holds
\begin{equation}
\label{dualrep} \rho(Z)=\sup_{\zeta \in \Ac} \langle \zeta, Z\rangle =\sup_{dQ/dP\in \Ac}\E_Q [Z].
\end{equation}
This result reveals how measures of risk provide robustness with respect to the changes of the distribution. Their application constitutes a new approach to robust statistical inference.

For a random variable $X\in \Lc_p (\varOmega)$ with distribution
function $F_X(\eta)=P\{X \le \eta\} $, we consider the survival function
\[
\bar{F}_X(\eta) = P(X>\eta)
\]
and the left-continuous inverse of the cumulative distribution function defined as follows:
\[
 F^{(-1)}_X(\alpha) = \inf\ \{ \eta : F_X(\eta) \ge \alpha \} \quad\text{for}\quad 0 < \alpha < 1.
\]
It is clear that $F^{(-1)}_X(\alpha)$ is the left $\alpha$-quantile of $X$. 

We intend to apply the theory to investigate the distribution of classification errors and that is why we have a preference to small outcomes (small errors).
We define the \emph{Value at Risk} at level $\alpha$ of a random error $X$ by setting
\[
{\rm VaR}_{\alpha}(X) = F^{(-1)}_X(1-\alpha),
\]
which implies that 
\[
P(V> {\rm VaR}_{\alpha}(X)) \leq \alpha.
\]
The risk here is defined as the probability of the error $X$ obtaining a large value. For a given $\alpha$, we can minimize  the value at risk by appropriately selecting the parameters of the classifier. This point of view corresponds to minimizing the probability of misclassification. 
Although Value at Risk is intuitively appealing measure, it is not coherent.

In the theory of measures of risk a special role is played by the functional called the Average Value-at-Risk and denoted ${\rm AVaR}(\cdot)$ (see \cite{acerbi2002coherence,ogryczak2002dual,rockafellar2002conditional}). The \emph{Average Value at Risk} of $X$ at level $\alpha$ is defined as
\begin{equation}
\label{eq:AVAR}
{\rm AVaR}_{\alpha}(X) = \frac{1}{\alpha} \int_0^{\alpha} {\rm VaR}_t(X)\, dt.
\end{equation}
Consider the integrated survival function of the random variable $X$,
\[
\bar{F}^{(2)}_X(\eta) = \int_\eta^\infty \bar{F}_X(t)\big]\; dt=\mathbb{E}[(X-\eta)_+].
\]
The second equality is shown in \cite{ddgmtwo}. 
The \emph{upper Lorenz function} $\bar{F}^{(-2)}_X:\R\to\overline{\R}$ is introduced in \cite{ddgmtwo} as a counterpart of the absolute Lorenz function 
(cf. \cite{Lorenz:1905,arnold1987lorenz,gastwirth1972estimation}). It is defined as follows:
\begin{equation}
\label{ulf}
\bar{F}^{(-2)}_X(\alpha)= \int_{\alpha}^1F_X^{-1}(t)\,dt\quad \text{for } 0< \alpha < 1.
\end{equation}
Additionally, $\bar{F}^{(-2)}_X(1)=0$, $\bar{F}^{(-2)}_X(0)=\mathbb{E}(X)$, and $\bar{F}_X^{(-2)}(\alpha)= -\infty$  for $\alpha\not\in[0,1]$.
The function $\bar{F}^{(2)}_X(\cdot)$ is concave because its derivative is monotonically non-increasing.
It is shown in \cite{ddgmtwo} that the Fenchel conjugate function of  the integrated survival function $\bar{F}^{(2)}_X(\cdot)$ is the function
$-\bar{F}^{(-2)}_X(\cdot+1)$. 
This statement is a counterpart of the conjugate duality relation for the absolute Lorenz curve, which has been first established in
 \cite[Theorem 3.1]{OgRu:2002}.
From the definition of the upper Lorenz function, we obtain that it represents the Average Value-at-Risk: 
 \begin{equation}
\label{ulf-cont}
\bar{F}^{(-2)}_X(1-\alpha)= \int_{1-\alpha}^1{\rm VaR}_{1-t}(X)\,dt = \int_0^{\alpha}{\rm VaR}_{\beta}(X)\,d\beta\quad \text{for } 0< \alpha < 1.
\end{equation}
We obtain
\[
{\rm AVaR}_\alpha(X)= \frac{1}{\alpha} \bar{F}^{(-2)}_X(1-\alpha) \quad \text{for } 0< \alpha < 1.
\]
Thus, using the conjugate duality relation from  \cite{ddgmtwo}, we obtain 
\begin{align*}
{\rm AVaR}_\alpha(X) & = \frac{1}{\alpha} \bar{F}^{(-2)}_X(1-\alpha)
 = -\frac{1}{\alpha} \sup_\eta\Big\{ -\alpha\eta- \mathbb{E}[\max(0,X-\eta)]\Big\}\\
& =\inf_\eta\Big\{ \eta + \frac{1}{\alpha} \mathbb{E}[\max(0,X-\eta)]\Big\}.
\end{align*}
This is the representation (cf. also \cite{shapiro2014lectures}) suitable for optimization problems. 

Due to Kusuoka theorem (\cite{kusuoka2001law},\cite[Thm. 6.24]{shapiro2014lectures}), every law invariant, finite-valued
coherent measure of risk on $\Lc^{p}(\varOmega)$ for non-atomic probability space can be represented as a mixture of Average Value-at-Risk at all probability levels.
This result can be extended for finite probability spaces with equally likely observations.
Kusuoka representations allows to extend statistical estimators of Lorenz curves to spectral law-invariant measures of risk as shown in \cite{dentcheva2010shape}. Central limit theorems for general composite risk functionals is established in \cite{dentcheva2016statistical}.

Other popular coherent measures of risk (when small outcomes are preferred) include the upper mean-semi-deviations of order $p$, defined as
\begin{equation}
\label{ro-semi}
\sigma^+_p[Z] := \E[Z] + \kappa\Big( \E\Big[\big(Z-\E[Z]\big)_+^p\Big]\Big)^{1/p},
\end{equation}
where $p\in[1,\infty)$ is a fixed parameter. It is well defined for all random variables $Z$ with finite $p$-th order moments and is coherent for $\kappa\in [0,1]$.
In the special case of $p=1$, the upper semi-deviation is equal to 1/2 of the absolute deviation, i.e., 
\[
\E\Big[\big(Z-\E[Z]\big)_+\Big] = \frac{1}{2} \E\Big[\big|Z-\E[Z]\big|\Big]
\]
Other classes of coherent measures of risk were proposed and analyzed in \cite{cheridito2009risk,dentcheva2010kusuoka,krokhmal2007higher,ogryczak2002dual,shapiro2014lectures} and the references therein.

In \cite{rockafellar2008risk}, the use of coherent measures of risk for generalized regression and model fit was proposed.
This point of view was also utilized in SVM in the report \cite{gotoh2017support}.
While those works recognize the need of expressing different attitude to errors in fitting statistical models, the authors propose using one overall measure of risk as an objective in the regression problem, respectively in the SVM problem.
The classification design based on a single measure of risk does not allow for differentiation between the classes.
Our point of view is that \emph{different attitude should be allowed to classification errors for the different classes}.


\section{Risk Sharing Preliminaries} 
\label{sec:risk_sharing_preliminaries}

The notion of risk sharing and analysis of this topic is a subject of intensive investigations in the community of economics, quantitative finance and risk management. This is due to the fact that the sum of the risk of each component in a system does not equal the risk of the entire system.
Risk allocation assumes that there is a quantitative assessment undertaken by a higher authority
within a firm, which  divides the firm's costs between the constituents.
The main focus in the extant literature on risk-sharing is on the choice of decomposition of a random variable $X$ into $k$ terms $X=X^1+\dots + X^k$, so that when each component is measured by a specific risk measure, the associated total risk is in some sense optimal. The variable $X$ represents the total random loss of the firm and the question addressed is about splitting the loss among the  constituents.
Assigning coherent measures of risk $\varrho_i$ to each term $X^i$, the adopted point of view is that the outcome $\big(\varrho_1(X^1), \dots, \varrho_k(X^k)\big)$ should be Pareto-optimal among the feasible allocations.

The main results in risk-sharing theory accomplish the decomposition of $X$ into terms by looking at the infimal convolution of the measures of risk, which is defined as follows. Given convex functions $f_i:\R^n\to \R$, $i=1,\dots k$, their \emph{infimal convolution} is the function 
$f_{1}\Box \cdots \Box f_{k} :\R^n\to \R$  (see,\cite[p. 57]{Rock-duality}) defined by
\[
[f_{1}\Box \cdots \Box f_{k}] (x) = \inf \{ f_{1}(x_1)+\dots + f_{k}(x_k) \,:\; x_1+\dots +x_k=x \}.
\]
The infimal convolution is a convex function and its Fenchel-conjugate satisfies is the sum of the conjugate function $f_i^*$,  $i=1,\dots, k$, i.e.,
\[
[f_{1}\Box \cdots \Box f_{k}]^{*}=f_{1}^{*}+\cdots +f_{k}^{*}.
\]
The risk-sharing problem amounts to the evaluation of the infimal convolution 
\[
[\varrho_{1}\Box \cdots \Box \varrho_{k}] (X).
\]
It is observed (see, e.g., \cite{landsberger1994co,ludkovski2008comonotonicity}) that the random variables $X^i$, $i=1,\dots, k$, which solve this problem, satisfy a co-monotonicity property as follows
\[
\big(X^i(\omega) -X^i(\omega')\Big)\big(X^j(\omega) -X^j(\omega')\Big) \geq 0,\quad\text{for all } \omega,\omega'\in\varOmega,\;\; i,j=1,\dots,k.
\]
We shall discuss the optimality of a risk allocation decision in due course. At the moment, 
we note that the problem setting and the results associate with risk sharing of losses in financial institutions are inapplicable to the classification problem. We cannot expect co-monotonicity properties of the class errors because not all decomposition of the total random error can be obtained via some classifier. The presence of constraints in the optimization problem, the functional dependence of the misclassification error on the classifier's parameters, and the complex nature of design problem require dedicated analysis.


\section{Risk Sharing in Classification} 
\label{sec:risk_sharing_in_classification}

If the distribution of the vectors $X^i$, $i=1,\dots, k$, are known, then the optimal risk-neutral classifier would be obtained by minimizing the expected error. 
This would be the solution of the following optimization problem:
\begin{equation}
\label{p:general-E}
\begin{aligned}
\min\, &\, \sum_{i=1}^k \mathbb E\big[ Z^i(\pi)\big]\\
\text{subject to } &\, Z^i(\pi) = \dist(\varphi(X^i;\pi),{K_i}),\quad i=1,\dots, k,\\
&\, \pi\in\mathcal D.
\end{aligned}
\end{equation}

We shall introduce the notion of a risk-averse classifier. Let a set of labeled data, a parametric classifier family
$\varphi(\cdot;\pi)$ with the associated collection of sets $K_i$, $i=1\dots,k$, and the law-invariant coherent risk measures $\varrho_i$, $i=1\dots,k$ be given. 
The presumption is that we have different attitude to misclassification risk in the various classes and the total risk is shared among the classes according to risk-averse preferences. 

We assume throughout that the set of feasible parameters $\pi$ is a closed convex set $\mathcal D\subseteq \mathcal R^s.$
Let $\Yc$ denote the set of all random vectors $(Z^1(\pi),\dots,Z^k(\pi))$ obtained as $Z^i(\pi)=\dist(\varphi(X^i;\pi),{K_i})$, i.e., 
$\Yc$ is the set of \emph{all attainable classification errors} considered as random vectors in the corresponding probability space. 
In the classification problem, we deal with their representation from the available sample calculated as follows:
\[
z^i_j(\pi)=\dist(\varphi(x_j;\pi),{K_i}), x_j\in S_i,\quad j=1,\dots, m_i,\;\; i=1,\dots k.
\]
for a given parameter $\pi\in \mathcal D$.
\begin{definition}
A vector $w\in \R^k$ represents an \emph{attainable risk allocation} for the classification problem, if a parameter $\pi\in\mathcal D$ exists such that 
\[
w = \big(\varrho_1(Z^1(\pi)), \dots, \varrho_k(Z^k(\pi))\big)\in\R^k\quad\text{for}\quad\big(Z^1(\pi),\dots,Z^k(\pi)\big)\in \Yc.
\]
\end{definition}
We denote the set of all attainable risk allocations by $\Xc$.
Assume that a partial order on $\R^k$ is induced by a pointed convex cone $\Kc\subset\R^k$, i.e.,
\[
v\preceq_\Kc w \text{ if and only if } w-v\in \Kc.
\]
Recall that a point $v\in A\subset \R^k$ is called $\Kc$-\emph{minimal point of the set} $A$ if no point $w\in A$ exists such that $v-w\in\Kc$.
If $\Kc=\R^k_+$, then the notion of $\Kc$-minimal points of a set corresponds to the well-known notion of Pareto-efficiency or Pareto-optimality in $\R^k$.
\begin{definition}
A classifier $\varphi(\cdot;\pi)$ is called $\Kc$-optimal risk-averse classifier, if its risk-allocation is a $\Kc$-minimal element of $\Xc$.
If $\Kc=\R^k_+$, then the classifier is called Pareto-optimal.
\end{definition}
From now on, we focus on Pareto-optimality, but our results are extend-able to the case of more general orders defined by pointed cones.
\begin{definition}
A \emph{risk-sharing classification problem} (RSCP) is given by the set of labeled data, a parametric classifier family
$\varphi(\cdot;\pi)$ with the associated collection of sets $K_i$, $i=1\dots,k$, and a set of law-invariant risk measures $\varrho_i$, $i=1\dots,k$. 
The risk-sharing classification problem consists of identifying a parameter $\pi\in\mathcal D$ resulting in a Pareto-optimal classifier $\varphi(\cdot;\pi)$.
\end{definition}
We shall see that the Pareto-minimal risk allocations are produced by random vectors, which are minimal points in the set $\Yc$ with respect to the usual stochastic order, defined next.
\begin{definition}
\label{d:order}
A random variable $Z$ is stochastically larger than a random variable $Z'$ with respect to the usual stochastic order (denoted $Z\fsd Z'$), if
\begin{equation}
\label{e:def-fsd}
\mathbb P  (Z >\eta)\geq \mathbb P  (Z' >\eta)\quad \forall\, \eta\in\R,
\end{equation}
or, equivalently, $F_{Z}(\eta)\leq F_{Z'}(\eta)$.  

The relation is strict (denoted $Z\sfsd Z'$), if additionally, inequality \eqref{e:def-fsd} is strict for some $ \eta\in\R$.

A random vector $\mathbf{Z}= (Z_1,\dots Z_k)$ is stochastically larger than $\mathbf{Z}'= (Z'_1,\dots Z'_k)$
(denoted $\mathbf{Z} \succeq \mathbf{Z}'$) if
$Z_i\fsd Z_i'$ for all $i=1,\dots k$. The relation is strict if for some component $Z_i\sfsd Z_i'$.
\end{definition}
The random vectors of $\Yc$, which are non-dominated with respect to this order will be called \emph{minimal points of $\Yc$}.

For more information on stochastic orders see, e.g., \cite{shaked2007stochastic}.

The following result is known for non-atomic probability spaces. We verify it for a sample space in order to deal with the empirical distributions.
\begin{theorem}
\label{t:consistency} Suppose  the probability space $(\Omega,\Fc,P)$ is finite with equal probabilities of all simple events.  Then every law-invariant risk
functional $\rho$ is consistent with the usual stochastic order if and only if it satisfies the monotonicity
axiom. If $\rho$ is strictly monotonic with respect to the almost sure relation, then $\rho$ is consistent with the strict dominance relation, i.e. $\rho(Z_1)<\rho(Z_2)$ whenever $Z_2\sfsd Z_1$.
\end{theorem}
\begin{proof}
Assuming that $\Omega=\{ \omega_1,\dots,\omega_m\}$, let the random variable $U(\omega_i)=\frac{i}{m}$ for all $i=1,\dots, m$.
If $Z_2\fsd Z_1$, then defining $\hat{Z}_1:=F_{Z_1}^{-1}(U)$ and $\hat{Z}_2:=F_{Z_2}^{-1}(U)$,
we obtain $\hat{Z}_2(\omega)\geq \hat{Z}_1(\omega)$ for all $\omega\in \Omega$. Due to the monotonicity axiom, $\rho(\hat{Z_2}) \ge \rho(\hat{Z_1})$. The random variables $\hat{Z}_i$ and $Z_i$,  $i=1,2$, have the same distribution by construction. This entails that $\rho({Z_2}) \ge \rho({Z_1})$ because the risk measure is law invariant.
Consequently, the risk measure $\rho$ is consistent with the usual stochastic order. The other direction is straightforward.
\end{proof}
This observation justifies our restriction to risk measures, which are consistent with the usual stochastic order, also known as the first order stochastic dominance relation. Furthermore, when dealing with non-negative random variables as in the context of classification, then strictly monotonic risk measures associate no risk only when no misclassification occurs, as shown by the following statement.
\begin{lemma}
If $\rho$ is a law invariant strictly monotonic coherent measure of risk, then
\begin{equation}
\label{a:norm}
\begin{aligned}
\rho(Z) &> 0 \text{ for all random variables } Z\geq 0\,\, a.s., Z\not\equiv 0\\
\rho(Z) & < 0  \text{ for all random variables } Z\leq 0\,\, a.s., Z\not\equiv 0.
\end{aligned}
\end{equation}
\end{lemma}
\begin{proof}
Denote the random variable, which is identically equal zero by $\mathbf{0}$. Notice that $\rho(\mathbf{0}) = \rho(2\cdot\mathbf{0}) = 2\rho(\mathbf{0})$, which implies that  $\rho(\mathbf{0})=0$.
If $Z\geq 0$ a.s. and $Z\not\equiv 0$, then $\rho(Z)> \rho(\mathbf{0})=0$ by the strict monotonicity of $\rho$.
The second statement follows analogously.
\end{proof}
 This statement implies that  $\varrho_i(Z^i(\pi))\geq 0$ for all $\pi\in\mathcal D$ and for all $i=1,\dots k$  and, therefore, the attainable allocations lie in the positive orthant, i.e.,   $\Xc\subseteq \mathbb R^k_+$.

From now on, we dopt the following assumptions:
\begin{description}
  \item[(A1)] The risk measures $\rho^i$ used for evaluation of classification errors in classes $i=1,...,k$
are coherent, law invariant, and finite-valued.
\item[(A2)] The sets $K_i$, $i=1,\dots k$ and  $\Dc\subseteq\R^s$ are non-empty, closed and convex. 
\end{description}
that  
\begin{theorem}
\label{t:continuity}
Assume (A1), (A2) and let the support of the random vectors $X^i$,  $i=1,\dots k$, be bounded. 
If the function $\varphi(x,\cdot)$ is continuous for every argument $x\in\R^n$, then the components of the attainable risk allocations 
$\rho_i(Z^i(\cdot))$, $i=1,\dots k$, are continuous functions.
If additionally, each component of the vector function $\varphi(x,\cdot)$ is an affine function, then $\rho_i(Z^i(\cdot))$, $i=1,\dots k$ are convex functions.
\end{theorem}
\begin{proof}
 The distance functions $z\mapsto \dist(z,K_i)$ are continuous convex functions (see, e.g., \cite{beer1993topologies}) and $\dist(z,K_i)<\infty$ for all $z\in\R^n$. 
 Thus, the composition of the distance function with the continuous function $\varphi(x;\cdot)$ is continuous, meaning that the random variable 
 $Z^i (\pi) = \dist(\varphi(X^i;\pi),K_i)$ has realizations, which are continuous functions of $\pi$. Furthermore, the variables $Z^i$ have bounded support due to the boundedness assumption of the theorem.  Therefore, $Z^i (\cdot)$ is continuous with respect to the norm in the space $\mathcal{L}_p(\varOmega)$. 
 Since the risk measures $\rho_i(\cdot)$ are convex and finite, they are continuous on $\mathcal{L}_p$ for $p\geq 1$.  
 We conclude that its composition with the risk measure: $\rho_i(Z^i (\cdot))$, is continuous.

In order to prove convexity, let $\lambda\in (0,1)$ and let $\pi_\lambda = \lambda\pi +(1-\lambda)\pi'$. 

Let $z^i(\pi), z^i(\pi')\in K_i$ be the points such that 
\begin{gather}
\| \varphi(x;\pi) -z^i(\pi)\|= \min_{z\in K_i}\| \varphi(x;\pi) -z\|\\
 \| \varphi(x;\pi) -z^i(\pi')\|= \min_{z\in K_i}\| \varphi(x;\pi') -z\|
\end{gather}
We define $z_\lambda = \lambda z^i(\pi)+(1-\lambda) z^i(\pi')$. Due to the convexity of $K_i$, we have $z_\lambda\in K_i.$
As $\varphi(x,\cdot)$ is affine, we obtain
\[
\varphi(x;\pi_\lambda) = \lambda\varphi(x;\pi) +(1-\lambda)\varphi(x;\pi'). 
\]
This entails the following inequality for all $i=1,\dots k$ and all $z\in\R^d$:
\begin{align*}
\min_{z\in K_i}\| \varphi(x;\pi_\lambda) -z\| & \leq  \| \varphi(x;\pi_\lambda) -z^i_\lambda\| =  \| \varphi(x;\pi_\lambda) -\lambda z^i(\pi)-(1-\lambda) z^i(\pi')\| \\
& = \| \lambda\big(\varphi(x;\pi)-z^i(\pi)\big) +(1-\lambda)\big(\varphi(x;\pi')-z^i(\pi')\big)\big\| \\
&\leq 
\lambda\|\varphi(x;\pi)-z^i(\pi)\| +(1-\lambda)\|\varphi(x;\pi')-z^i(\pi'))\big\|\\
 & = \lambda\min_{z\in K_i}\| \varphi(x;\pi)-z\| +(1-\lambda)\min_{z\in K_i}(\varphi(x;\pi')-z)\big\|.
\end{align*}
Therefore,
\[
\dist(\varphi(x;\pi_\lambda),K_i)  \leq \lambda\dist(\varphi(x;\pi),K_i) +(1-\lambda)\dist(\varphi(x;\pi'),K_i).
\]
The monotonicity and convexity axioms for the risk measures imply that
\begin{multline*}
\rho_i \big(\dist(\varphi(X;\pi_\lambda),K_i) \big)
\leq \lambda \rho_i\big(\dist(\varphi(X;\pi),K_i) \big) 
+ (1-\lambda)\rho_i\big(\dist(\varphi(X;\pi'),K_i)\big).
\end{multline*}
\end{proof}
This result implies the existence of Pareto-optimal classifier. Furthermore, the convexity property allows us to identify the Pareto-optimal risk-allocations by using scalarization techniques.
\begin{corollary}
\label{c:exists}
  Assume (A1), (A2) and let the function $\varphi(x,\cdot)$ be affine for every argument $x\in\R^n$. Then a parameter $\pi$ defines a Pareto-optimal classifier $\varphi(\cdot, \pi)$ for the given RSCP if and only if a scalarization vector $w\in\R^k_+$ exists with $\sum_{i=1}^k w_i=1$, such that $\pi$ is a solution of the problem
\begin{equation}
\label{scalarproblem}
\min_{\pi\in\Dc} \sum_{i=1}^k w_i\rho_i\big(\dist(\varphi(X_i;\pi),K_i)\big).
\end{equation}
\end{corollary}
\begin{proof}
Statement follows form the well-known scalarization theorem in vector optimization problems (\cite{miettinen1999nonlinear})  and Theorem~\ref{t:continuity}. 
\end{proof}
\begin{theorem}
Assume that the risk measures $\rho_i$ are law invariant and strictly monotonic for all $i=1,\dots k.$
If a classifier $\varphi(\cdot;\pi)$ is Pareto-optimal, then its corresponding random vector $(Z^1(\pi),\dots, Z^k(\pi))$ is a minimal point of $\Yc$  with respect to the order of Definition \ref{d:order}.
\end{theorem}
\begin{proof}
Suppose that $\varphi(\cdot;\pi)$ is Pareto-optimal and the point \\
$Z(\pi) =(Z^1(\pi),\dots, Z^k(\pi))$ is not minimal. Then a parameter $\pi'$ exists, such that the corresponding vector $Z(\pi')$ is strictly stochastically dominated by  $Z$, which implies $Z^i(\pi) \fsd Z^i(\pi') $ with a strict relation for some component. We obtain
$\rho_i(Z^i(\pi)) \geq \rho_i(Z^i(\pi'))$ for all $i=1,\dots, k$ with a strict inequality for some $i$ due to the consistency of the coherent measures of risk with the strong stochastic order relation, which contradicts the Pareto-optimality of $\varphi(\cdot;\pi)$.
\end{proof}
We consider the sample space $\varOmega=\prod_{i=1}^k \varOmega_i$ where $(\varOmega_i,\Fc_i,P_i)$ is a finite space with $m_i$ simple events $\omega_j\in\varOmega_i$, $P_i(\omega_j)=\frac{1}{m_i}$, and $\Fc_i$ consisting of all subsets of $\varOmega_i$.
\begin{theorem}
\label{t:conditions}
Assume (A1) and (A2). Suppose each component of the vector function $\varphi(x,\cdot)$ is affine for every $x\in\R^n$.
If the parameter $\hat{\pi}$  defines a Pareto-optimal classifier $\varphi(\cdot, \hat{\pi})$ for the RSCP, then a probability measure $\mu$ on $\varOmega$ exists so that $\hat{\pi}$ is an optimal solution for the problem
\begin{equation}
\label{p:risk-neutral}
\min_{\pi\in\Dc} \sum_{i=1}^k \sum_{j=1}^{m_i} \mu^i_j \dist(\varphi(x^i_j;\pi),K_i).
\end{equation}
\end{theorem}
\begin{proof}
Since the parameter $\hat{\pi}$  defines a Pareto-optimal classifier $\varphi(\cdot, \hat{\pi})$ for the RSCP and all conditions of Corollary~\ref{c:exists} are satisfied, then $\hat{\pi}$ is an optimal solution of problem \eqref{scalarproblem} for some scalarization $w$.
Let $\Ac_i$ denotes the set of probability measures corresponding to the risk measure $\rho_i$ , $i=1,\dots, k$ in representation
\eqref{dualrep}. 
Since the risk measures $\rho_i$ take finite values on $\varOmega_i$, the sets $\Ac_i$ are non-empty and compact. Thus, the supremum in the dual representation \eqref{dualrep} is achieved at some elements $\zeta^i\in\Ac_i$. We have  $\zeta^i_j\geq 0$, $\sum_{j=1}^{m_i} \frac{\zeta^i_j}{m_i} =1$ because 
$\zeta_i$ are probability densities. We obtain
\[
\rho_i(\dist(\varphi(X^i;\pi),K_i)) = \sum_{j=1}^{m_i} \frac{\zeta^i_j}{m_i} \dist(\varphi(x^i_j;\pi),K_i).
\]
Setting
\[
\mu^i_j=w_i\frac{\zeta^i_j}{m_i},\; j=1,\dots, m_i,\; i=1,\dots, k
\]
we observe that the vector $\mu\in R^{m_1+\dots m_k}$ constitutes a probability mass function.
Thus, problem \eqref{scalarproblem} can be reformulated as \eqref{p:risk-neutral}.
\end{proof}

This result shows that the RSCP can be viewed as a classification problem in which the expectation error is minimized. However, the expectation is not calculated with respect to the empirical distribution but with respect to another measure $\mu$, which is implicitly determined by the chosen measures of risk. 
It is the worst expectation according to our risk-averse preferences, which are represented by the choice of the measures $\rho_i$, $i=1,\dots, k.$

The composite nature of the problem \eqref{scalarproblem} is difficult and that is why we reformulate the problem. We introduce auxiliary variables 
$Y\in\mathcal L_p(\varOmega, \mathcal F, P;\R^{m})$, $i=1,\dots k$, which are defined by the constraints:
\[
\varphi(X^i;\pi)+Y^i \in K_i \quad \forall i=1,\dots, k. 
\]
Problem \eqref{scalarproblem} can be reformulated to
\begin{equation}
\label{p:split}
\begin{aligned}
\min_{\pi,Y}\, &\,\sum_{i=1}^k w_i\varrho_i(\|Y^i\|)\\
\text{s.t. }  &\, \varphi(X^i;\pi)+Y^i\in K_i, \quad \forall i=1,\dots, k,\\
&\, \pi\in\Dc.
\end{aligned} 
\end{equation}
We shall show that this problem is equivalent to \eqref{scalarproblem}. 
\begin{lemma}
For any solution $\hat{\pi}$ of problem \eqref{scalarproblem}, random vectors $\hat{Y}^i$ exist, so that $(\hat{\pi},\hat{Y})$ solves problem \eqref{p:split} as well,
where $\hat{Y}= (\hat{Y}^k,\dots, \hat{Y}^k)$ and for any solution $(\hat{\pi},\hat{Y})$ of problem \eqref{p:split}, the vector $\hat{\pi}$ is a solution of problem \eqref{scalarproblem} as well.
\end{lemma}
\label{l:equiv}
\begin{proof}
Observe that for any fixed point $\pi\in\Dc$, the function $\sum_{i=1}^k w_i\varrho_i(\|Y^i\|)$ achieves minimal value with respect to the constraints on the variables $Y^i$ using the projections of the realizations of $X^i$ onto $K_i$:
\begin{equation}
\label{hatY}
Y^i (\omega) = {\rm Proj}_{K_i}\big((\varphi(X(\omega);{\pi})\big)  - \varphi(X(\omega);{\pi}). 
\end{equation}
Here ${\rm Proj}_{K_i}(z)$ denotes the Euclidean projection of the point $z$ onto the set $K_i.$
Then, $\|Y^i \|= \dist(\varphi(X^i;\pi),K_i)$ and the objective functions of both problems have the same value. 
Therefore, the minimal value is achieved at the same point $\hat{\pi}$ and the corresponding  $\hat{Y}^i_j$ is obtained from
equation \eqref{hatY}.
\end{proof}
Recall that the normal cone to a set $\Dc\subset \R^s$ is defined as 
\[
\mathcal{N}_{\Dc}(\pi) = \{a\in\R^s: \langle a, d - \pi\rangle \leq 0\;\text{ for all } d\in\Dc\}.
\]
For brevity, we denote the normal cone to the feasible set of problem \eqref{p:split} by $\mathcal{N}$ and the normal cones to the sets $K_i$ by
$\mathcal{N}_i$, $i=1,\dots, k$. 
We formulate optimality conditions for problem \eqref{p:split}.

We denote the realizations of the random vectors $Y^i$, $i=1,\dots, k$, by $y^i_j(\pi)$, $j=1,\dots m_i$, $i=1,\dots, k$. More precisely, we have
\[
y^i_j(\pi) = {\rm Proj}_{K_i}\big((\varphi(x^i_j;{\pi})\big)  - \varphi(x^i_j;{\pi})\quad j=1,\dots m_i,\; i=1,\dots, k. 
\]
We suppress the argument $\pi$ whenever it does not lead to confusion.
Additionally, we denote the Jacobian of $\varphi$ with respect to $\pi$ by $ D\varphi(x;{\pi})$. 
Consider the sample-based version of problem  \eqref{p:split}:
\begin{equation}
\label{p:split-sample}
\begin{aligned}
\min_{\pi,Y}\, &\,\sum_{i=1}^k w_i\varrho_i(\|Y^i\|)\\
\text{s.t. }  &\, \varphi(x^i_j;\pi)+y^i_j\in K_i, \quad \forall j=1,\dots, m_i,\; i=1,\dots, k,\\
&\, \pi\in\Dc.
\end{aligned} 
\end{equation}
\begin{theorem}
\label{t:opt-split-sample} 
Assume that the sets $K_i$, $i=1,\dots, k$ are closed convex polyhedral cones and $\varphi(x; \cdot)$ is an affine vector function. 
A feasible point $(\hat{\pi},\hat{Y})$ is optimal for problem \eqref{p:split-sample} if and only if probability mass functions  $\zeta^i\in \partial \rho_i (0)$ and vectors  $g^i_j$ from $\partial \|\hat{y}^i_j\|$ exist such that
\begin{align}
0\in & -\sum_{i=1}^k \sum_{j=1}^{m_i} w_i \zeta^i_j(g^i_j)^\top  D\varphi(X^i;\hat{\pi}) +\mathcal N_{\Dc}(\hat{\pi})\label{kkt-pi-f}\\
w_i\zeta^i_j g^i_j &\in \mathcal N_i\big(\varphi(x^i_j;\hat{\pi})+\hat{y}^i_j \big) \text{ for all }\; j=1,\dots m_i,\; i=1,\dots k. \label{compl-f}
\end{align}
\end{theorem} 
\begin{proof}
We assign Lagrange multipliers $\lambda^i_j$ to the inclusion constraints and define the Lagrange function as follows:
\[
L(\pi, Y,\lambda) = \sum_{i=1}^k \Big( w_i\varrho_i(\|Y^i\|) + \sum_{j=1}^{m_i} \big\langle \varphi(x^i_j;\pi)+{y}^i_j , \lambda^i_j\big\rangle \Big).
\]
Using optimality conditions \cite[Theorem 3.4]{BonnansShapiro}, we obtain that $(\hat{\pi},\hat{Y})$ is optimal for problem \eqref{p:split-sample} if and only if
$\hat{\lambda}$ exists such that
\begin{gather*}
0\in \partial_{(\pi,Y)} L(\hat{\pi},\hat{Y},\hat{\lambda}) +\mathcal{N}(\hat{\pi},\hat{Y})\\
\hat{\lambda}^i_j \in \mathcal N_i \big(\varphi(x^i_j;\hat{\pi}) +\hat{y}^i_j) \big).
\end{gather*} 
Considering the partial derivatives of the Lagrangian with respect to the two components, we obtain
\begin{gather}
0\in  \sum_{i=1}^k  \sum_{j=1}^{m_i}  (\hat{\lambda}^i_j)^\top D\varphi(x^i_j;\hat{\pi})+\mathcal N_{\Dc} (\hat{\pi})\label{kkt-pi}\\
0= w_i\partial_Y \rho_i (\|Y\|) +\hat{\lambda}^i,\quad i=1,\dots k,\label{kkt-Y}\\
\hat{\lambda}^i_j\in \mathcal N_i\big(\varphi(x^i_j;\hat{\pi}) +\hat{y}^i_j \big) , \quad j=1,\dots, m_i,\; i=1,\dots k. \label{compl}
\end{gather} 
We calculate the multipliers $\hat{\lambda}^i $ from the equation \eqref{kkt-Y} using elements $\zeta^i\in \partial \rho_i (0)$ and $g^i_j$ from
$\partial \|\hat{y}^i_j\|$. We obtain:
\[
\hat{\lambda}^i_j= - w_i \zeta^i_j g^i_j,\quad j=1,\dots, m_i,\; i=1,\dots k.
\]
Notice that $ g^i_j = \frac{\hat{y}^i_j}{\|\hat{y}^i_j\|}$ whenever $\hat{y}^i_j \ne 0$, otherwise $g^i_j\in\R^d$ can be any vector 
with $\|g^i_j\|\leq 1$. 
Substituting the value of $\hat{\lambda}^i $ into \eqref{kkt-pi} and \eqref{compl}, we obtain condition \eqref{kkt-pi-f} and \eqref{compl-f}. 
\end{proof}

We note that, we can define again a probability mass function $\mu$ by setting $\mu^i_j= w_i\zeta^i_j$ and interpret the Karush-Kuhn-Tucker condition
as follows: 
\begin{gather*}
\mathbb{E}_{\mu} (g^i_j)^\top  D\varphi(X^i;\hat{\pi}) \in \mathcal N_{\Dc}(\hat{\pi})\\
\mu^i_j g^i_j \in \mathcal N_i\big(\varphi(x^i_j;\hat{\pi})+\hat{y}^i_j \big) \text{ for all  }\; j=1,\dots m_i,\; i=1,\dots k.
\end{gather*}

Problem \eqref{p:split-sample} can be reformulated as a risk-averse two-stage optimization problem (cf. \cite{SDR}).
The first stage decision is $\pi$ and the first stage problem is 
\begin{equation}
\label{p:first-stage}
\min_{ \pi\in\Dc} \sum_{i=1}^k w_i\rho_i\big(Z^i(\pi))\big) .
\end{equation}
Given $\pi$, the calculation of each realization of $Z^i(\pi)$ amounts to solving the following problem 
\begin{equation}
\label{p-distance}
z^i_j(\pi)= \min_{y\in K_i} \| \varphi(x^i_j;\pi) -y\|, \quad j=1,\dots m_i,\;  i=1,\dots k.
\end{equation}
Calculating $z^i_j(\pi)$ might be very easy for specific regions $K_i$  such as the cones in the example of the polyhedral classifier.
Every component of the solution vector $\hat{z}^i_j$ to problem \eqref{p-distance} can be computed as follows:
\[
(\hat{z}^i_j)_\ell = \begin{cases} \max \{0, -(\varphi(x^i_j;\pi))_\ell\} &\text{ for } \ell= i;\\
 \max \{0, (\varphi(x^i_j;\pi))_\ell\} &\text{ for } \ell\ne i; 
 \end{cases}
 \quad \ell =1,\dots, k.
\]
Then the optimal value of \eqref{p-distance} is
\[
z^i_j(\pi) = \Big( \sum_{\ell=1}^k (\hat{z}^i_j)_\ell^2\Big)^{\frac{1}{2}}.
\]
This point of view facilitates the application of stochastic optimization methods to solve the problem.

\section{Confidence Intervals for the Risk}
\label{sec:confidence}

In this section, we analyze the risk-averse classification problem when we increase the data sets and derive confidence intervals for the misclassification risk. 
We use the results on statistical inference for composite risk functionals presented in \cite{dentcheva2016statistical}.
In \cite{dentcheva2016statistical}, a composite risk functional is defined in the following way.  
\begin{equation}\label{defcrf}
 \rho (X) = \mathbb{E} \left[ f_{1}\left( \mathbb{E} \left[  f_{2}\left( \mathbb{E} \left[ \cdots f_{\ell}\left(\mathbb{E}\left[ f_{\ell+1}\left( X \right)\right],X\right)\right]\cdots,X \right)\right],X\right)\right] 
\end{equation}
where $X$ is an $n-$dimensional random vector with unknown distribution, $P_{X}$. The functions $f_{j}$ are such that $f_{j}(\eta_{j},x):\mathbb{R}^{n_{j}}\times \mathbb{R}^{n} \rightarrow \mathbb{R}^{n_{j-1}}$ for $j = 1,\ldots,\ell$ and $n_{0} = 1$. The function $f_{\ell+1}$ is such that $f_{\ell+1}(x):\mathbb{R}^{n} \rightarrow \mathbb{R}^{n_{\ell}}$. 

A law-invariant risk-measure $\rho(X)$ is an unknown characteristic of the distribution $P_{X}$. The empirical estimate of $\rho(X)$ given $N$ independent and identically distributed observations of $X$ is given by the plug-in estimate
\begin{equation}\label{defcrfa} 
 \begin{aligned}
\rho^{(N)} = \sum_{i_0=1}^N\frac{1}{N}\Big[f_1\Big(\sum_{i_1=1}^N\frac{1}{N}\big[f_2\big(\sum_{i_2=1}^N\frac{1}{N}[&\;\cdots f_\ell(\sum_{i_\ell=1}^N\frac{1}{N}f_{\ell+1}(X_{i_\ell}),X_{i_{\ell-1}})]\\
&\;\cdots,X_{i_1}\big)\big],X_{i_0}\Big)\Big]
\end{aligned}
\end{equation}
 It is shown in \cite{dentcheva2016statistical} that the most popular measures of risk fit the structure \eqref{defcrf}. It is established that the plug-in estimator satisfies a central limit formula and the limiting distribution is described. This is the distribution of the Hadamard-directional derivative of the risk functional 
 $\rho$ when a normal random variable is plugged in. 
Recall the notion of Hadamard directional derivatives of the functions $f_{j}\big(\cdot,x)$ at points $\mu_{j+1}$ in directions
$\zeta_{j+1}$. It is given by 
\[
f'_{j}\big(\mu_{j+1},x;\zeta_{j+1}) = \lim_{{t\downarrow 0}\atop{s\to \zeta_{j+1}}}\frac{1}{t}
\big[ f_{j}\big(\mu_{j+1}+ts,x) - f_{j}\big(\mu_{j+1},x)\big].
\]
The central limit formula holds under the following conditions:
\begin{itemize}
\item[(i)]  $\int \| f_j(\eta_j,x)\|^2 \;P(dx)<\infty$ for all $\eta_j\in I_j$, and $\int \dist^2(\varphi(X^i;\pi),{K_i}) P(dx)<\infty$;
\item[(ii)]  For all realizations $x$ of $X^i$, the functions $f_j(\cdot,x)$, $j=1,\dots,\ell$, are Lipschitz continuous:
\[
\|f_j(\eta_j',x)- f_j(\eta_j'',x)\| \le \gamma_j(x) \|\eta_j'-\eta_j''\|,\quad \forall\; \eta_j',\eta_j'',
\]
and $\int \gamma_j^2(x)\;P(dx) <\infty$.
\item[(iii)]  For all realizations $x$ of $X^i$, the functions $f_j(\cdot,x)$, $j=1,\dots,\ell$, are Hadamard directionally differentiable.
\end{itemize}

These properties are satisfied for the mean-semideviation risk measures as shown in \cite{dentcheva2016statistical}. Furthermore, it is shown that similar construction represents the Average-Value-at-Risk. 

For every parameter $\pi$ the risk of misclassification for a given class $i=1,\dots, k$ can be fit to the setting \eqref{defcrf} by choosing 
the innermost function $f_{\ell+1}(x):\mathbb{R}^{d} \rightarrow \mathbb{R}$ to be  $f_{\ell+1}(x) = \dist(\varphi(x;\pi),{K_i})$ whenever $\varphi$ satisfies properties i--iii.

In our setting each misclassification risk $\varrho_i\Big(\dist\big(\varphi(X^i;\pi),{K_i}\big)\Big)$ is estimated by $\varrho_i^{(m_i)}\big(\|\hat{Y}^i\|\big)$, where $(\hat{Y}^i;\hat{\pi})$ is the solution of problem \eqref{p:split-sample}. 
Denoting the estimated variance of the limiting distribution of $\varrho_i^{(m_i)}\big(\|\hat{Y}^i\|\big)$ (briefly $\varrho_i^{(m_i)}$) by  $\sigma_i^2$, we obtain the following confidence interval:
\[
\Big[ \;\rho_i^{(m_i)} - t_{\alpha,{\rm df}}\frac{\sigma_i}{\sqrt{m_i}},\quad 
\varrho_i^{(m_i)} + t_{\alpha,{\rm df}}\frac{\sigma_i}{\sqrt{m_i}}\;\;
\Big].
\]
Here $\alpha$ is the desired level of confidence, $t_{\alpha,{\rm df}}$ is the corresponding quantile of the t-distribution with degrees of freedom $df$. 
The degrees of freedom depend on the choice if risk measure and can be calculated as 
$df=m_i-\ell$, where $\ell$ is the number of compositions in formula \eqref{defcrfa}. 
The decrease of the degrees of freedom form $m_i$ is due to the estimation of the expected value associated with each composition.
The total risk is estimated by 
\[
 \hat{\rho}= \sum_{i=1}^k w_i\varrho_i^{(m_i)}\big(\|\hat{Y}^i\|\big).
\]
We obtain that $\hat{\rho}$ has an approximately normal distribution with expected value
$\rho$ and variance $\sum_{i=1}^k \frac{w_i^2\sigma_i^2}{m_i}.$ A confidence interval for $\rho$ is given by
\[
\left[\;\hat{\rho}- t_{\alpha,{\rm df}}\sqrt{\sum_{i=1}^k \frac{w_i^2\sigma_i^2}{m_i}},\quad 
\hat{\rho} +  t_{\alpha,{\rm df}} \sqrt{\sum_{i=1}^k \frac{w_i^2\sigma_i^2}{m_i}}\;\;
\right] .
\]


\section{Risk Sharing in SVM} 
\label{sec:risk_sharing_in_svm}

We analyze the SVM problem in more detail.
We consider only law-invariant strictly monotonic coherent measures of risk $\varrho_1, \varrho_2$ for the two classes $S_1$ and $S_2$.

The \emph{risk-sharing SVM problem (RSSVM)} consists in identifying a parameter $\pi=(v,\gamma)\in\R^n$ corresponding to a Pareto-minimal point of the attainable risk-allocation $\Xc$ for the affine classifier $\varphi(z;\pi)=\langle v, z\rangle -\gamma$.
Due to Corollary`\ref{c:exists}, we can determine a risk-averse classifier by solving the following problem:

\begin{equation}
\label{p:RSSVM-pure}
\begin{aligned}
\min_{v,\gamma,Z^1,Z^2}\ &\,  \lambda\varrho_1(Z^1) + (1-\lambda)\varrho_2(Z^2)\\
\text{s. t. } &\, \langle v, x^1_j\rangle - \gamma + z_j^1 \ge 0, \quad j=1,\dots,m_1,\\
&\, \langle v, x^2_{j}\rangle -\gamma  - z_{j}^2 \le 0, \quad j=1,\dots,m_2,\\
&\, \langle v,v\rangle =1,\\
&\, Z^1\ge 0,\, Z^2\geq 0.
\end{aligned}
\end{equation}
Here $\lambda\in (0,1)$ is a parameter and 
the vectors $Z^i$ have realization $z_j^i$, $i=1,2$ and $j=1,\dots, m_i$, representing the classification error for the sample of each class. The random vectors $Z^i$ can be represented by a deterministic vectors stacking all realizations $z_j^i$ as components (sub-vectors) of it. Abusing notation, we shall use $Z^i$ also for those long vectors in $\mathbb R^{nm_i}$.  

We note that  the normalization of  the vector $v$ automatically bounds $\gamma$ because for any fixed $v$,
the component $\gamma$ can be considered restricted in a compact set $[\gamma_m(v), \gamma_M(v)] $, where 
\[
\gamma_M = \max_{1\leq j\leq m_i,\; i=1,2} v^\top x^i_j\quad  \gamma_m =\min_{1\leq j\leq m_i,\; i=1,2} v^\top x^i_j.
\]
Thus, in this case, we can set $\Dc=\R^n.$
We also consider a soft-margin risk-averse SVM based on problem \eqref{eq:svmE1}, although the classification error might not be calculated properly.
The problem reads
\begin{equation}
\label{p:RSSVM}
\begin{aligned}
\min_{v,\gamma,Z^1,Z^2}\ &\,  \lambda\varrho_1(Z^1) + (1-\lambda)\varrho_2(Z^2) +\delta\|v\|^2\\
\text{s. t. } &\, \langle v, x^1_j\rangle - \gamma + z_j^1 \ge 1, \quad j=1,\dots,m_1,\\
&\, \langle v, x^2_{j}\rangle -\gamma  - z_{j}^2 \le -1, \quad j=1,\dots,m_2,\\
&\, Z^1\ge 0,\, Z^2\geq 0.
\end{aligned}
\end{equation}
In this problem, $\delta>0$ is a small number. The objective function grows to infinity when the norm of $v$ increases.  Thus, we do not need to bound the norm of the vector $v$. It also automatically bounds $\gamma$, similar to problem \eqref{p:RSSVM-pure}.

We observe that the parameter $(v,\gamma)$ for each Pareto-optimal classifier can be obtained by solving the following problem:
\begin{equation}
\label{p:weighted}
\begin{aligned}
\min_{v,\gamma,Z^1,Z^2}\ &\,  \varrho_1(Z^1) + \varrho_2(Z^2) \\
\text{s. t. } &\, \langle v, x^1_i\rangle - \gamma + \frac{1}{\lambda}z_i^1 \ge 0, \quad i=1,\dots,m_1,\\
&\, \langle v, x^2_{j}\rangle -\gamma  - \frac{1}{1-\lambda}z_{j}^2 \le 0, \quad j=1,\dots,m_2,\\
&\, \langle v,v\rangle =1,\\
&\, Z^1\ge 0,\, Z^2\geq 0.
\end{aligned}
\end{equation}
\begin{lemma}
Problem \eqref{p:weighted} is equivalent to problem \eqref{p:RSSVM-pure}.   
\end{lemma}
\begin{proof}
The equivalence follows from the axiom of positive homogeneity for the risk measures:
\[
\lambda \rho_1(Z^1) = \rho_1(\lambda Z^1) \quad\text{and}\quad (1-\lambda) \rho_2(Z^2) = \rho_2((1-\lambda) Z^2).
\]
Defining new random variables $\tilde{Z}^1= \lambda Z^1$ and $\tilde{Z}^2= (1-\lambda) Z^2$, we can rescale the variables in their
respective inequality constraint.  
\end{proof}
This observation is a counterpart of the result in \cite{jouini2008optimal} for the risk sharing of random losses among constituents. \\

\section{Numerical Experiments} 
\label{sec:numerical_experiments}

In the previous sections, we have shown the solid theoretical foundation supporting our approach.
In this section, we display the performance of the proposed framework, as well as its flexibility.
To this end, we use several publicly available data sets and compare the performance of our approach to some existing formulations, in terms of F$_1$--score. Further, we showcase the flexibility of the framework by exploring the Pareto-efficient frontier of various classifiers derived from our framework.
In our numerical experiments, we have used the Average Value-at-Risk and the mean semi-deviation of order one.

\subsection{Data} 
\label{sub:data}

We compare our approach to other known approaches on several datasets. 
More specifically, we use three data sets obtained from the UCI Machine Learning Repository \cite{Lichman:2013}. 
These data sets exhibit different degrees of class imbalance, that is the proportion of records in one class versus that of the other class. 
A summary of basic characteristics of the data sets is shown in the following table.

\begin{table}[h!]
\centering
\begin{tabular}{lcccl}
  \hline
  \multirow{2}{*}{\textbf{Data Set}}& \multirow{2}{*}{Features}& \multicolumn{2}{c}{Observations} & Class \\
   &  & Class0 & Class1 (\%)                & Balance\\
  \hline
  wdbc     & 30 &  357 & 211 (37.1)  & 0.591 \\
  pima-indians-diabetes &  7 &  500 & 267 (34.8)  & 0.534 \\
  seismic-bumps  & 18 & 2414 & 170 ( 6.6)  & 0.070 \\
  \hline
\end{tabular}
\caption{Data summary}
\label{tbl:data_summary}
\end{table}


\subsection{Model Formulations} 
\label{sub:model_formulations}

We consider several scenarios for choices of measures of risk.  
In the first scenario, we treat one of the classes (Class0) in a risk neutral manner, while applying the mean-semi-deviation measure to the classification error of the second class. 
We call this loss function ``asym\_risk'' (see Table \ref{tbl:formulations}). 
In the same table, we provide the risk measure combinations for other loss functions which we have used in our numerical experiments.
The loss functions called ``risk\_cvar'' and ``two\_cvar'' use a convex combination of the expected error and the Average Value-at-Risk of the classification error.
These convex combinations use an additional model parameter $\beta \in (0,1)$. 
We note that such a convex combination is a coherent measure of risk.
The formulation \eqref{p:RSSVM} for these loss function require modification due to the use of the variational form of the Average-Value at Risk at level $\alpha\in (0,1)$.
Table \ref{tbl:formulations} displays the chosen combinations of risk measure pairs for the binary classification scenario in order to give an easy overview.

  \begin{table}[!h]
    \centering
    \small
    \begin{tabu}{lcc}
      \textbf{Loss Function} &  Class0 -- $\varrho_1(Z^1)$ & Class1 -- $\varrho_2(Z^2)$ \\
      \tabucline[1pt]{-}
      exp\_val    & $\E[Z^1]$ & $\E[Z^2]$ \\
      joint\_cvar  & \multicolumn{2}{c}{$\beta\E[Z^1 + Z^2] + (1-\beta){\rm AVaR}_{\alpha}(Z^1 + Z^2)$ }\\
      \hline
      asym\_risk  & $\E[Z^1]$ & $\E[Z^2] + {c} \sigma^+[Z^2]$ \\
      one\_cvar   & $\E[Z^1] + {c} \sigma^+[Z^1]$  & ${\rm AVaR}_{\alpha}(Z^2)$ \\
      risk\_cvar  & $\E[Z^1] + {c} \sigma^+[Z^1]$ & $\beta\E[Z^2] + (1-\beta){\rm AVaR}_{\alpha}(Z^2)$ \\
      two\_risk   & $\E[Z^1] + {c} \sigma^+[Z^1]$ & $\E[Z^2] + {c} \sigma^+[Z^2]$ \\
      two\_cvar   & $\beta\E[Z^1] + (1-\beta){\rm AVaR}_{\alpha_1}(Z^1)$ & $\beta\E[Z^2] + (1-\beta_2){\rm AVaR}_{\alpha_2}(Z^2)$ \\
      \tabucline[1pt]{-}
    \end{tabu}
    \caption{Risk measure combinations used as loss functions in the experiments}
    \label{tbl:formulations}
  \end{table}
We note that calculation of the first order semi-deviation and the average value at risk can be formulated as linear optimization problems. Therefore, their  application does not increase the complexity of RSSVM in comparison to the soft-margin SVM. 
However, if we use higher order semi-deviations or higher order inverse risk measures, the problem becomes more difficult. 

We compare our results against three different benchmarks: two risk-neutral formulations and one risk-averse formulation with a single risk measure.
The first risk-neutral formulation is the soft-margin SVM as formulated in \eqref{eq:svmE1}. 
The second risk-neutral formulation uses the Huber loss function and leads to the following problem formulation
\begin{equation}
\label{p:svm-huber}
\begin{aligned}
\min_{v,\gamma,Z^1,Z^2}\ &\, \frac{1}{m_1} \sum_{i=1}^{m_1} \min\big(z_i^1,(z_i^1)^2\big) + \frac{1}{m_2} \sum_{j=1}^{m_2} \min\big(z_j^2,(z_j^2)^2\big) +\delta\|v\|^2\\
\text{s. t. } &\, \langle v, x^1_i\rangle - \gamma + z_i^1 \ge 1, \quad i=1,\dots,m_1,\\
&\, \langle v, x^2_{j}\rangle -\gamma  - z_{j}^2 \le -1, \quad j=1,\dots,m_2,\\
&\, Z^1\ge 0,\, Z^2\geq 0.
\end{aligned}
\end{equation}

The third benchmark uses a single risk measure \eqref{p:svm-jointCVaR} on the total error as proposed in \cite{gotoh2017support}. It has the following formulation.
\begin{equation}
\label{p:svm-jointCVaR}
\begin{aligned}
\min_{v,\gamma,t,Z^1,Z^2, Y^1,Y^2}\ &\,
 \beta\Big( \frac{1}{m_1}\sum_{j=1}^{m_1} z^1_j + \frac{1}{m_2}\sum_{j=1}^{m_2} z^2_j \Big) + \\
 &\,\qquad(1-\beta)\bigg( t + \frac{1}{\alpha(m_1+m_2)}\Big( \sum_{j=1}^{m_1} y^1_j + \sum_{j=1}^{m_2} y^2_j \Big) \bigg) +\delta\|v\|^2\\
\text{s. t. } &\, \langle v, x^1_i\rangle - \gamma + z_i^1 \ge 1, \quad i=1,\dots,m_1,\\
&\, \langle v, x^2_{j}\rangle -\gamma  - z_{j}^2 \le -1, \quad j=1,\dots,m_2,\\
&\, y_j^i\geq z_j^i-t,\quad j=1,\dots,m_i,\; i=1,2,\\
&\, Z^1\ge 0,\, Z^2\geq 0,\; Y^1\ge 0,\, Y^2\geq 0.
\end{aligned}
\end{equation}

Interestingly, both risk-neutral formulations produce identical results on all data sets.
Subsequently we only report one of them under the name ``exp\_val''. 
In the presented figures and tables below, we refer to the loss function consisting of a single Average Value-at-Risk measure, as ``joint\_cvar''.

The problem formulations which we use in our experiments are the following.
\begin{description}

  \item[Expected value vs. Average Value-at-Risk -- ``asym\_risk'']
    \begin{equation}
      \label{p:RSSVM-E-cvar}
      \begin{aligned}
      \min_{v,\gamma,t,Z^1,Z^2,Y}\quad  &\,  \frac{\lambda}{m_1}\sum_{j=1}^{m_1} z^1_j+ 
                \frac{1-\lambda}{m_2}\sum_{j=1}^{m_2} (y_j+z^2_j) +\delta\|v\|^2 \\
      \text{s. t. }\quad  &\, \langle v, x^1_j\rangle - \gamma + z_j^1 \ge 1, \quad j=1,\dots,m_1, \\
      &\, \langle v, x^2_{j}\rangle -\gamma  - z_{j}^2 \le -1,\quad j=1,\dots ,m_2,\\
      &\, y_j \geq z_j^2-t,\quad j=1,\dots ,m_2,\\
      &\, Z^1\ge 0,\, Z^2\geq 0,\; Y\ge 0. 
      \end{aligned}
    \end{equation}

  \item[Mean-semi-deviation vs. Average Value-at-Risk -- ``one\_cvar'']
    \begin{equation}
      \label{p:RSSVM-msd-cvar}
      \begin{aligned}
      \min_{v,\gamma,t,Z^1,Z^2,Y^1,Y^2}\quad  &\,  \frac{\lambda}{m_1}\sum_{j=1}^{m_1} (y^1_j+z^1_j) + (1-\lambda)\big(t + \frac{1}{\alpha m_2}\sum_{j=1}^{m_2} y^2_j\big) +\delta\|v\|^2 \\
      \text{s. t. }\quad  &\, \langle v, x^1_j\rangle - \gamma + z_j^1 \ge 1, \quad j=1,\dots,m_1, \\
      &\, \langle v, x^2_{j}\rangle -\gamma  - z_{j}^2 \le -1, \quad j=1,\dots ,m_2,\\
      &\, y_j^1 \geq z_j^1-\frac{1}{m_1}\sum_{j=1}^{m_1} z^1_j,\quad j=1,\dots ,m_1,\\
      &\, y_j^2 \geq z_j^2-t,\quad j=1,\dots ,m_2,\\
      &\, Z^1\ge 0,\, Z^2\geq 0,\; Y^1\ge 0,\, Y^2\geq 0. 
      \end{aligned}
    \end{equation}

  \item[Mean-semi-deviation vs. combination of the expectation and AVaR -- ``risk\_cvar'']
    \begin{equation}
      \label{p:RSSVM-msd-Ecvar}
      \begin{aligned}
      \min_{v,\gamma,t,Z^1,Z^2,Y^1,Y^2}\quad  &\,   \frac{\lambda}{m_1}\sum_{j=1}^{m_1} (y^1_j+z^1_j) + \frac{\beta(1-\lambda)}{m_1}\sum_{j=1}^{m_2} z_j^2\\
              &\, \qquad  + (1-\beta)(1-\lambda)\left(t + \frac{1}{\alpha m_2}\sum_{j=1}^{m_2} y^2_j\right) 
              +\delta\|v\|^2 \\
      \text{s. t. }\quad  &\, \langle v, x^1_j\rangle - \gamma + z_j^1 \ge 1, \quad j=1,\dots,m_1, \\
      &\, \langle v, x^2_{j}\rangle -\gamma  - z_{j}^2 \le -1, \quad j=1,\dots,m_2, \\
      &\, y_j^1 \geq z_j^1-\frac{1}{m_1}\sum_{j=1}^{m_1} z^1_j,\quad j=1,\dots ,m_1,\\
      &\, y_j^2 \geq z_j^2-t,\quad j=1,\dots ,m_2,\\
      &\, Z^1\ge 0,\, Z^2\geq 0,\; Y^1\ge 0,\, Y^2\geq 0. 
      \end{aligned}
    \end{equation}

  \item[Mean-semi-deviation for both classes -- ``two\_risk'']
    \begin{equation}
      \label{p:RSSVM-twomsd}
      \begin{aligned}
      \min_{v,\gamma,Z^1,Z^2,Y^1,Y^2}\quad  &\,  \frac{\lambda}{m_1}\sum_{j=1}^{m_1} (y^1_j+z^1_j) + \frac{1-\lambda}{m_2}\sum_{j=1}^{m_2} (y^2_j+z^2_j) +\delta\|v\|^2 \\
      \text{s. t. }\quad  &\, \langle v, x^1_j\rangle - \gamma + z_j^1 \ge 1, \quad j=1,\dots,m_1, \\
      &\, \langle v, x^2_{j}\rangle -\gamma  - z_{j}^2 \le -1, \quad j=1,\dots,m_2, \\
      &\, y_j^i \geq z_j^i-\frac{1}{m_i}\sum_{j=1}^{m_i} z_j^i,,\quad j=1,\dots ,m_i,\; i=1,2,\\
      &\, Z^1\ge 0,\, Z^2\geq 0,\; Y^1\ge 0,\, Y^2\geq 0.
      \end{aligned}
    \end{equation}
  \item[Average-Value at Risk for both classes -- ``two\_cvar'']
    \begin{equation}
      \label{p:RSSVM-twocvar}
      \begin{aligned}
      \min_{v,\gamma,t_1,t_2,Z^1,Z^2,Y^1,Y^2}\quad  &\,  \delta\|v\|^2 + \lambda \beta_1 \sum_{j=1}^{m_1} z^1_j + \lambda(1-\beta_1)\left(t_1 + \frac{1}{\alpha m_1}\sum_{j=1}^{m_1} y^1_j\right) \\
      &\, \quad + (1-\lambda)\beta_2 \sum_{j=1}^{m_1} z^2_j +  (1-\lambda)(1-\beta_2)\left(t_2 + \frac{1}{\alpha m_2}\sum_{j=1}^{m_2} y^2_j\right)  \\
      \text{s. t. }\quad  &\, \langle v, x^1_j\rangle - \gamma + z_j^1 \ge 1, \quad j=1,\dots,m_1, \\
      &\, \langle v, x^2_{j}\rangle -\gamma  - z_{j}^2 \le -1,\quad j=1,\dots,m_2, \\
      &\, y_j^i \geq z_j^i-t_i,\quad j=1,\dots ,m_i,\; i=1,2,\\
      &\, Z^1\ge 0,\, Z^2\geq 0,\; Y^1\ge 0,\, Y^2\geq 0. 
      \end{aligned}
    \end{equation}
\end{description}


\begin{table*}[h!]
\centering
\small
\begin{tabu}{r|rrrrrrr}
  \multicolumn{8}{l}{$F_1$-score Optimized Classifiers} \\
  \hline
 \multicolumn{2}{r}{exp\_val} & joint\_cvar & asym\_risk & one\_cvar & risk\_cvar & two\_risk & two\_cvar \\
  \hline
  lambda &  &  & 0.70 & 0.57 & 0.56 & 0.60 & 0.64 \\ 
  alpha\_1 &  &  &  &  &  &  & 0.62 \\ 
  alpha\_2 &  & 0.55 &  & 0.88 & 0.75 &  & 0.62 \\ 
  C0 Errors & 21 & 17 & 16 & 13 & 11 & 15 & 12 \\ 
  C1 Errors & 15 & 11 & 11 & 10 & 9 & 9 & 9 \\
  FPR & 0.05882 & 0.04762 & 0.04482 & 0.03641 & 0.03081 & 0.04202 & 0.03361 \\ 
  Recall & 0.92925 & 0.94811 & 0.94811 & 0.95283 & 0.95755 & 0.95755 & 0.95755 \\ 
  Precision & 0.90367 & 0.92202 & 0.92627 & 0.93953 & 0.94860 & 0.93119 & 0.94419 \\ 
  $F_1$-score & 0.91628 & 0.93488 & 0.93706 & 0.94614 & \textbf{0.95305} & 0.94419 & 0.95082 \\ 
  AUC & 0.97904 & 0.98426 & 0.98569 & \textbf{0.98764} & 0.98535 & 0.98442 & 0.98451 \\ 
  \hline
  \multicolumn{8}{l}{} \\
  \multicolumn{8}{l}{AUC Optimized Classifiers} \\
  \hline
 \multicolumn{2}{r}{exp\_val} & joint\_cvar & asym\_risk & one\_cvar & risk\_cvar & two\_risk & two\_cvar \\
  \hline
  lambda &  &  & 0.43 & 0.57 & 0.69 & 0.37 & 0.42 \\ 
  alpha\_1 &  &  &  &  &  &  & 0.61 \\ 
  alpha\_2 &  & 0.65 &  & 0.88 & 0.66 &  & 0.61 \\ 
  C0 Errors & 21 & 21 & 18 & 13 & 14 & 23 & 16 \\ 
  C1 Errors & 15 & 13 & 11 & 10 & 13 & 12 & 13 \\ 
  FPR & 0.05882 & 0.05882 & 0.05042 & 0.03641 & 0.03922 & 0.06443 & 0.04482 \\ 
  Recall & 0.92925 & 0.93868 & 0.94811 & 0.95283 & 0.93868 & 0.94340 & 0.93868 \\ 
  Precision & 0.90367 & 0.90455 & 0.91781 & 0.93953 & 0.93427 & 0.89686 & 0.92558 \\ 
  $F_1$-score & 0.91628 & 0.92130 & 0.93271 & \textbf{0.94614} & 0.93647 & 0.91954 & 0.93208 \\ 
  AUC & 0.97904 & 0.98471 & 0.98697 & 0.98764 & 0.98776 & 0.98629 & \textbf{0.98922} \\ 
  \hline
  \end{tabu}
  \caption{Main results table for the WDBC dataset -- Displaying the model parameters for the each model formulation as well as the corresponding performance metrics.}
  \label{tbl:main-results_wdbc}
\end{table*}

\subsection{Performance} 
\label{sub:performance}


We perform $k$-fold cross-validation and all reported results are out of sample.
In Tables \ref{tbl:main-results_wdbc}, \ref{tbl:main-results_pima-indians-diabetes}, and \ref{tbl:main-results_seismic-bumps}, we report the F$_1$--score  and AUC, along with recall, precision, as well as false positive rate (FPR) for all loss functions. 
Additionally, we report the number of misclassified observations, as well as the chosen parameters where applicable.
In light of the fact that the F$_1$--score and AUC are competing metrics, for each dataset we present one of results results optimized for each metric.
We use this highlight the additional flexibility that the proposed method introduces, in the next section.

In the above Table \ref{tbl:main-results_wdbc}, we show the best value for each metric for each set in bold face.
We observe that for this particular dataset, the best performing model formulation with respect to the $F_1$-score is the ``risk\_cvar'' model; outperforming the risk neutral formulations by more than $0.04$.
On the other hand, if we consider the AUC to be the target metric, we notice the ``two\_cvar'' formulation has the highest value.
Further, we note that the ``one\_cvar'' model has the same parameters for both target metrics.
We find this to be unusual in our experiments.
While this formulation does not have the best value for the target metric, it too significantly outperforms the risk neutral formuations.

\begin{figure*}[h!]
\begin{center}
\includegraphics[width=0.95\linewidth]{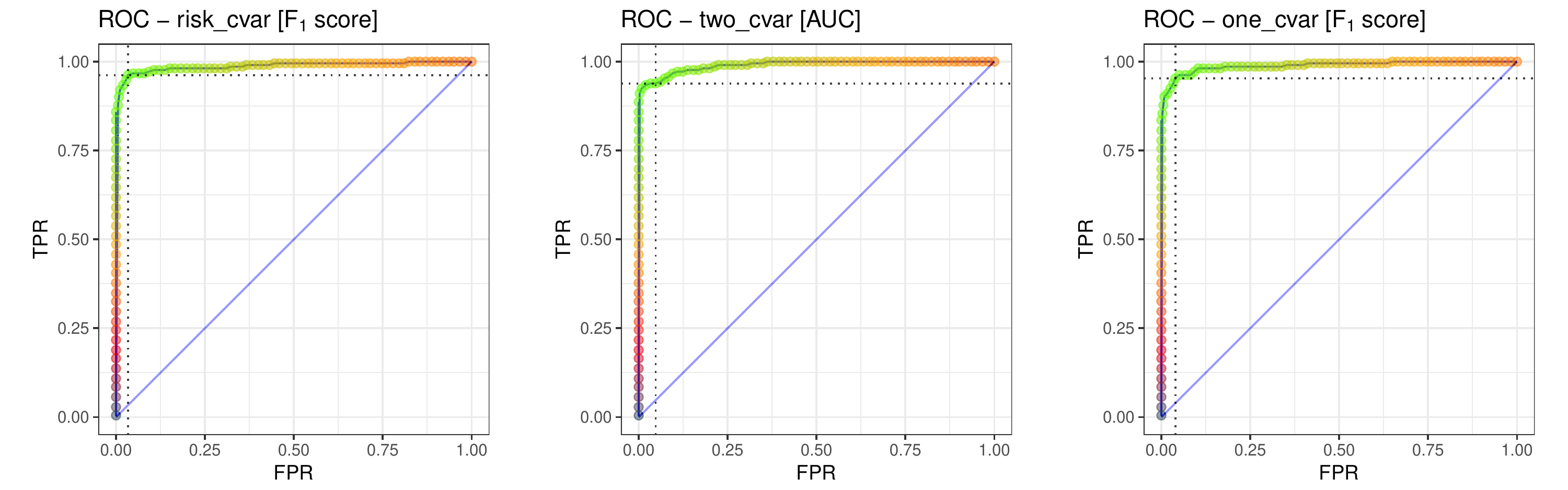}
\caption{ROC plots for the best performing model formulations on the WDBC data: ``risk\_cvar'' with the best $F_1$-score, ``two\_cvar'' with the best AUC value, and ``one\_cvar'' for the alternate metric.}
\label{fig:best-models-wdbc-roc}
\end{center}
\end{figure*}

Further, this formulation does have the best value for the competing metric in both cases.
The respective ROC curves for each of the classifiers are displayed in Figure \ref{fig:best-models-wdbc-roc}.
The color on each curve represents the value of the $F_1$-score.
High values are represented by the bright green color, and low values are represented by the dark red color.
The two dotted lines indicate the threshold at which the classifier is set to operate.

We can certainly see the classifier performs very well on this data. 
Table \ref{tbl:risk_eval-wdbc} contains the calulations of risk, with respect to each model formulation.
More specifically, for each obtained classifier we caculate the value of the risk functionals on the out of the sample data points during cross-validation.
We consider the raw expectation, mean semi-deviation, as well as the avarage value at risk for the $\alpha$ quantiles 0.75, 0.85, and 0.95.

\begin{table*}[h!]
\centering
\small
\begin{tabu}{rr|rrrrr}
  \multicolumn{7}{l}{WDBC} \\
  \tabucline[2pt]{-}
   \multicolumn{2}{r}{} & Expectation & MSD & AVaR$_{0.75}$ & AVaR$_{0.85}$ & AVaR$_{0.95}$ \\
  \hline
  \multirow{3}{*}{exp\_val}    & C0 Risk & 0.000189 & 0.000368 & 0.000252 & 0.000223 & 0.000199 \\ 
                               & C1 Risk & 0.000343 & 0.000663 & 0.000457 & 0.000403 & 0.000361 \\
                               & Total   & 0.000532 & 0.001030 & 0.000709 & 0.000626 & 0.000560 \\ 
                               \hline 
  \multirow{3}{*}{joint\_cvar} & C0 Risk & 0.000158 & 0.000309 & 0.000211 & 0.000186 & 0.000167 \\ 
                               & C1 Risk & 0.000241 & 0.000470 & 0.000322 & 0.000284 & 0.000254 \\ 
                               & Total   & 0.000400 & 0.000779 & 0.000533 & 0.000470 & 0.000421 \\ 
  \tabucline[1pt]{-}
  \multirow{3}{*}{asym\_risk}  & C0 Risk & 0.000121 & 0.000237 & 0.000161 & 0.000142 & 0.000127 \\ 
                               & C1 Risk & 0.000194 & 0.000378 & 0.000259 & 0.000228 & 0.000204 \\
                               & Total   & 0.000315 & 0.000615 & 0.000420 & 0.000371 & 0.000332 \\ 
                               \hline 
  \multirow{3}{*}{one\_cvar}   & C0 Risk & 0.000085 & 0.000166 & 0.000113 & 0.000100 & 0.000089 \\ 
                               & C1 Risk & 0.000172 & 0.000335 & 0.000229 & 0.000202 & 0.000181 \\ 
                               & Total   & 0.000256 & 0.000501 & 0.000342 & 0.000302 & 0.000270 \\ 
                               \hline
  \multirow{3}{*}{risk\_cvar}  & C0 Risk & 0.000080 & 0.000157 & 0.000106 & 0.000094 & 0.000084 \\ 
                               & C1 Risk & 0.000185 & 0.000363 & 0.000247 & 0.000218 & 0.000195 \\ 
                               & Total   & 0.000265 & 0.000520 & 0.000353 & 0.000312 & 0.000279 \\ 
                               \hline
  \multirow{3}{*}{two\_risk}   & C0 Risk & 0.000125 & 0.000246 & 0.000167 & 0.000148 & 0.000132 \\ 
                               & C1 Risk & 0.000182 & 0.000356 & 0.000242 & 0.000214 & 0.000191 \\ 
                               & Total   & 0.000307 & 0.000601 & 0.000410 & 0.000361 & 0.000323 \\ 
                               \hline
  \multirow{3}{*}{two\_cvar}   & C0 Risk & 0.000085 & 0.000167 & 0.000113 & 0.000100 & 0.000089 \\ 
                               & C1 Risk & 0.000235 & 0.000460 & 0.000314 & 0.000277 & 0.000248 \\ 
                               & Total   & 0.000320 & 0.000628 & 0.000427 & 0.000377 & 0.000337 \\ 
  \tabucline[2pt]{-}
\end{tabu}
\caption{Risk Evalutation for the WDBC data set -- Displaying the expectation of error, Mean Semi-deviation, and Avarage Value at Risk for the $\alpha$ quantiles 0.75, 0.85, and  0.95}
\label{tbl:risk_eval-wdbc}
\end{table*}

Indeed, we can observe that our models reduce the risk for each class with respect to each risk calculation, compared to the benchmarks.
More specifically, we notice that the ``one\_cvar'' model, which does not attain the best performance in terms of $F_1$-score, but does, in fact, attain the lowest total risk value. 
Its value is approximately one half that of the risk neutral formulation, and that of the other benchmark.
The ``risk\_cvar'' model does perform nearly identically, albeit having at slightly larger values across the board.
Further, we note that the ``two\_cvar'' model, which performes best with respect to the AUC metric is the worst performing, benchmarks excluded.
Looking closely at the corresponding ROC curve in Figure \ref{fig:best-models-wdbc-roc} one can argue that the performance with respect to the AUC metric, comes at the expense of robustness and generalization.

Looking at the results on the ``pima-indians-diabetes'' data set in Table \ref{tbl:main-results_pima-indians-diabetes} we observe that the best performing model with respect to $F_1$-score is the again ``risk\_cvar'' model with 0.68581 compared to the 0.66785 of the risk neutral formulations. 
Similarly, the ``one\_cvar'' model is again second in this conext, at the same time having the largest AUC value for the group.
Surprisingly, the benchmark formulation ``joint\_cvar'' has the lowest score here.
Switching the attention to the AUC section of the table, we notice that ``one\_cvar'' is the best performing model in that regard well; with the ``risk\_cvar'' being second best.
However, the gain in AUC value with the changed parameters is minimal with a considerable reduction in the alternate target metric; ``one\_cvar'' shifting from 0.68581 $F_1$-score to 0.65377 in exchange for 0.0027 gain in AUC, and ``risk\_cvar'' shifting from 0.68781 $F_1$ to 0.65504 for a gain of 0.003.

\begin{table*}[h!]
\centering
\small
\begin{tabu}{r|rrrrrrr}
  \multicolumn{8}{l}{$F_1$-score Optimized Classifiers} \\
  \hline
 \multicolumn{2}{r}{exp\_val} & joint\_cvar & asym\_risk & one\_cvar & risk\_cvar & two\_risk & two\_cvar \\
  \hline
  lambda &  &  & 0.48 & 0.51 & 0.49 & 0.48 & 0.44 \\ 
  alpha\_1 &  &  &  &  &  &  & 0.58 \\ 
  alpha\_2 &  & 0.90 &  & 0.68 & 0.56 &  & 0.58 \\ 
  C0 Errors & 107 & 92 & 158 & 121 & 125 & 129 & 157 \\ 
  C1 Errors & 80 & 93 & 46 & 65 & 62 & 63 & 48 \\ 
  FPR & 0.21400 & 0.18400 & 0.31600 & 0.24200 & 0.25000 & 0.25800 & 0.31400 \\ 
  Recall & 0.70149 & 0.65299 & 0.82836 & 0.75746 & 0.76866 & 0.76493 & 0.82090 \\ 
  Precision & 0.63729 & 0.65543 & 0.58421 & 0.62654 & 0.62236 & 0.61377 & 0.58355 \\ 
  $F_1$-score & 0.66785 & 0.65421 & 0.68519 & 0.68581 & \textbf{0.68781} & 0.68106 & 0.68217 \\ 
  AUC & 0.83039 & 0.83243 & 0.82900 & \textbf{0.83078} & 0.83033 & 0.82967 & 0.82830 \\ 
  \hline
  \multicolumn{8}{l}{} \\
  \multicolumn{8}{l}{AUC Optimized Classifiers} \\
  \hline
  \multicolumn{2}{r}{exp\_val} & joint\_cvar & asym\_risk & one\_cvar & risk\_cvar & two\_risk & two\_cvar \\
  \hline
  lambda &  &  & 0.51 & 0.54 & 0.54 & 0.50 & 0.60 \\ 
  alpha\_1 &  &  &  &  &  &  & 0.69 \\ 
  alpha\_2 &  & 0.59 &  & 0.86 & 0.76 &  & 0.69 \\ 
  C0 Errors & 107 & 87 & 140 & 80 & 79 & 113 & 74 \\ 
  C1 Errors & 80 & 98 & 59 & 99 & 99 & 78 & 106 \\ 
  FPR & 0.21400 & 0.17400 & 0.28000 & 0.16000 & 0.15800 & 0.22600 & 0.14800 \\ 
  Recall & 0.70149 & 0.63433 & 0.77985 & 0.63060 & 0.63060 & 0.70896 & 0.60448 \\ 
  Precision & 0.63729 & 0.66148 & 0.59885 & 0.67871 & 0.68145 & 0.62706 & 0.68644 \\ 
  $F_1$-score & 0.66785 & 0.64762 & \textbf{0.67747} & 0.65377 & 0.65504 & 0.66550 & 0.64286 \\ 
  AUC & 0.83039 & 0.83279 & 0.83081 & \textbf{0.83348} & 0.83332 & 0.83049 & 0.83267 \\ 
  \hline
  \end{tabu}
\caption{Main results table for the ``pima-indians-diabetes'' dataset -- Displaying the model parameters for the each model formulation as well as the corresponding performance metrics.}
\label{tbl:main-results_pima-indians-diabetes}
\end{table*}

Looking closely at the ROC curves in Figure \ref{fig:best-models-pima-indians-diabetes-roc} we can see that the AUC pioritized ``one\_cvar'' actually does not classify at its maximum potential in terms of $F_1$-score, indicated by the fact that the threshold is not at the lightest green segment of the curve.
This requires additional investigation and exploration.

\begin{figure*}[h!]
\begin{center}
\includegraphics[width=0.95\linewidth]{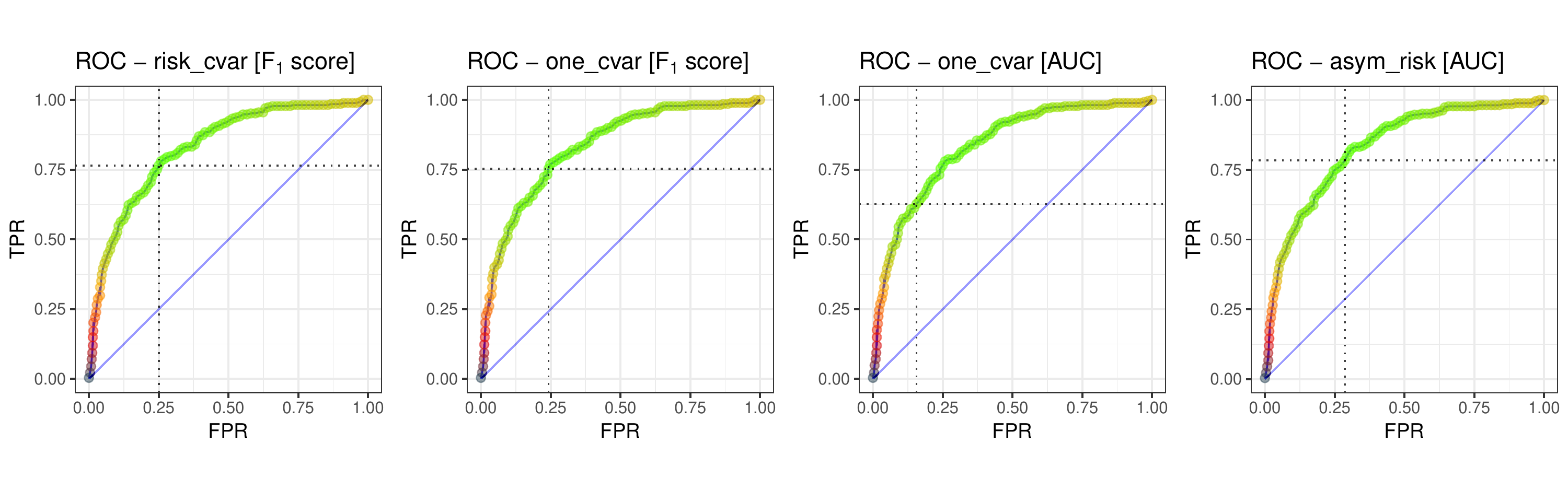}
\caption{ROC plots for the best performing model formulations on the ``pima-indians-diabetes'' data: ``risk\_cvar'' with the best $F_1$-score, ``one\_cvar'' featuring both parameter sets, and finally the ``asym\_risk'' formulation featuring the best AUC value}
\label{fig:best-models-pima-indians-diabetes-roc}
\end{center}
\end{figure*}

Figure \ref{fig:pima-indians-diabetes_cdf-roc_fscore} shows how the empirical distribution of error realizations from applying the classifier to out-of-sample records on the left, and the overlayed ROC curves for the various classifiers on the right. 
Negative values indicate correctly classified observations, while positive values indicate misclassification.
We compare the select loss functions to eachother and the benchmarks.
Virtually no distinction can be made between the ROC curves for the various classifiers.
However, looking at the error distribution plot on the left, we notice that the the two benchmarks misclassify less of the default class and more of the target class.
On the other hand, the ``two\_cvar'' formulation underperforms for the opposite reason, in relation to the target metric and the best performing formulation ``risk\_cvar''.

\begin{figure*}[!h]
\centering
\includegraphics[width=0.545\linewidth]{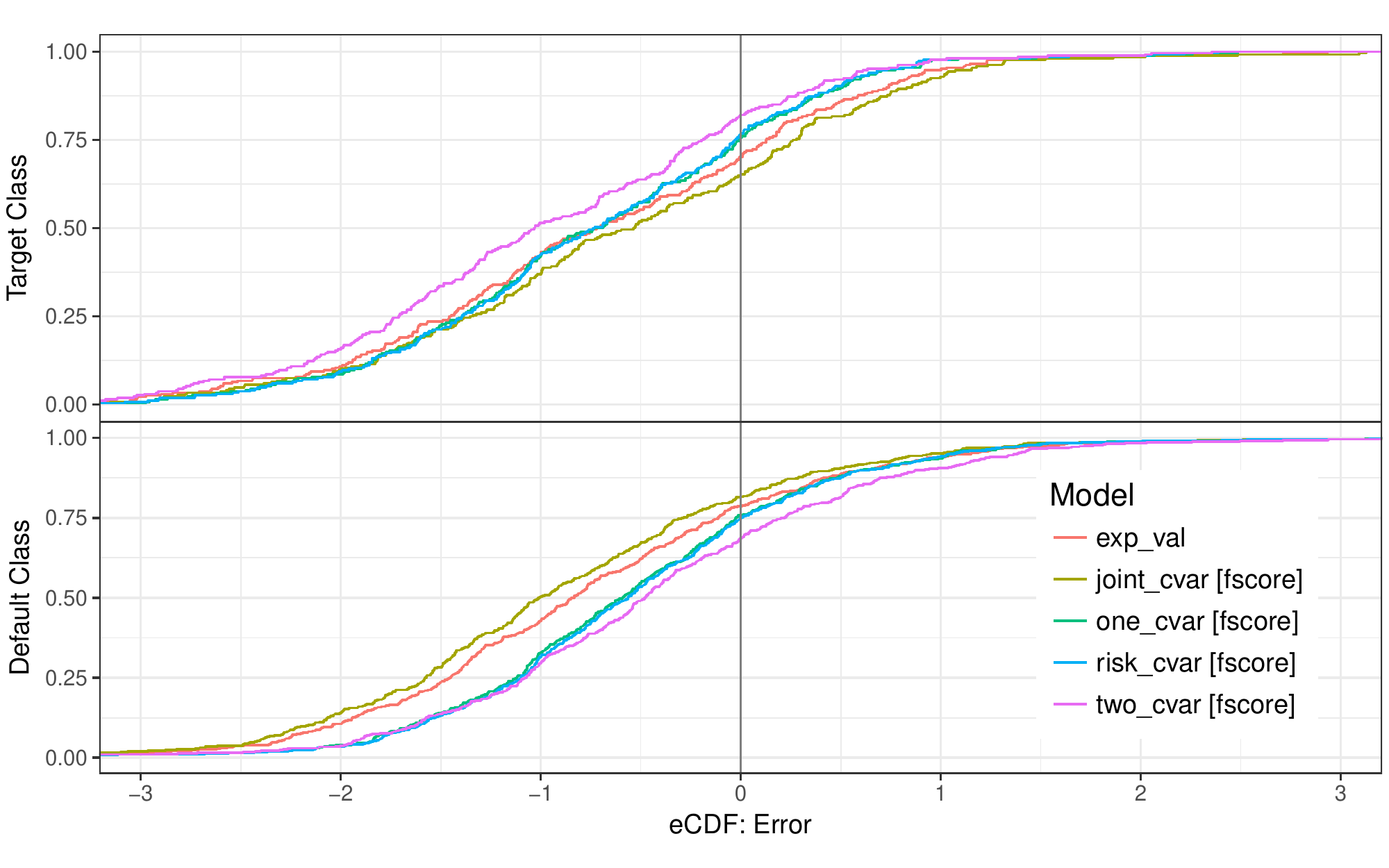}
\includegraphics[width=0.445\linewidth]{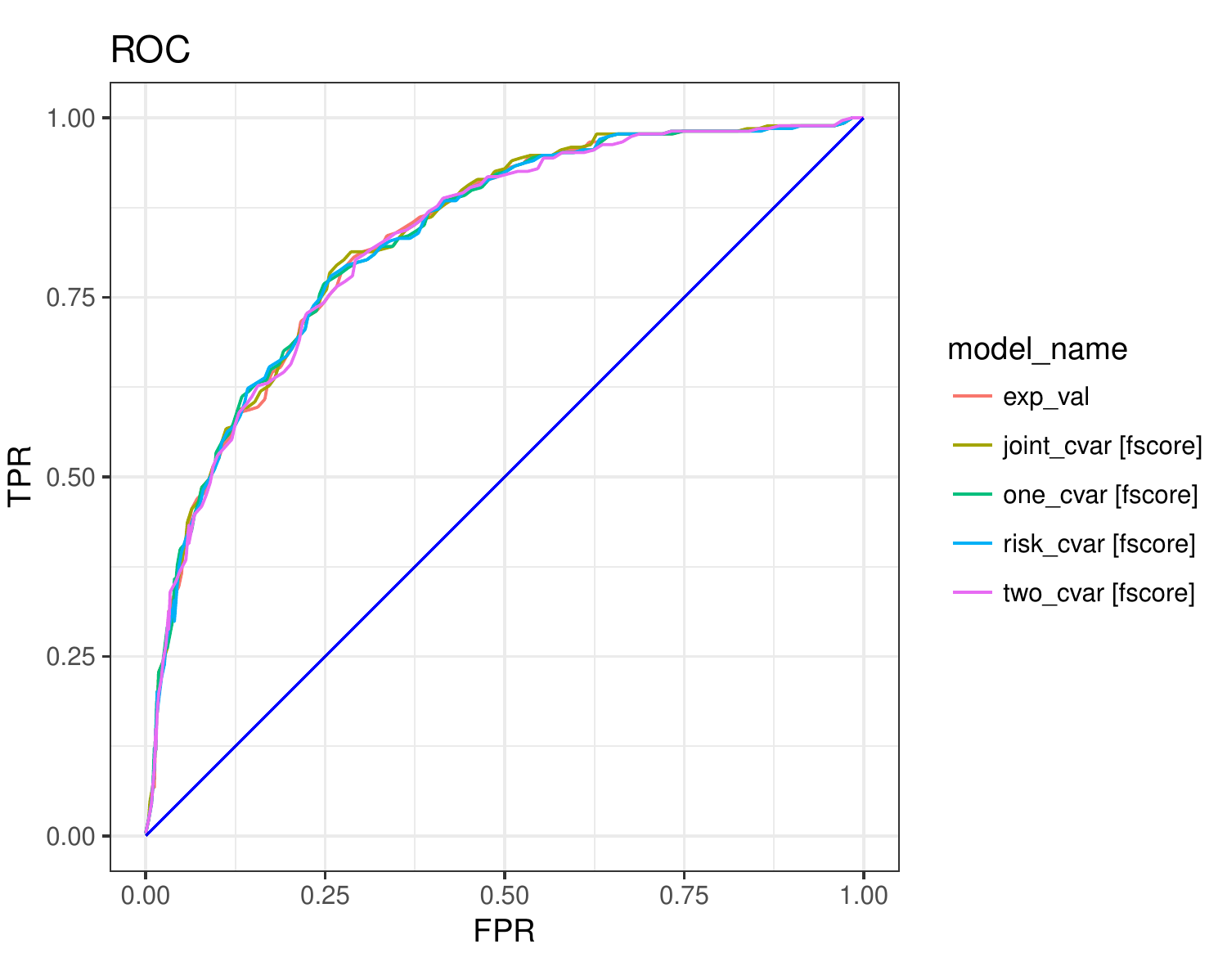}
\caption{Empirical distribution of error realizations comparing risk-averse loss function formulations to benchmarks [$F_1$-score] on the ``pima-indians-diabetes'' dataset (left) and the corresponding ROC curves (right)}
\label{fig:pima-indians-diabetes_cdf-roc_fscore}
\end{figure*}

\begin{table*}[h!]
\centering
\small
\begin{tabu}{rr|rrrrr}
  \multicolumn{7}{l}{pima-indians-diabetes} \\
  \tabucline[2pt]{-}
   \multicolumn{2}{r}{} & Expectation & MSD & AVaR$_{0.75}$ & AVaR$_{0.85}$ & AVaR$_{0.95}$ \\
  \hline
  \multirow{3}{*}{exp\_val}    & C0 Risk & 0.164317 & 0.296266 & 0.219089 & 0.193314 & 0.172965 \\ 
                               & C1 Risk & 0.183513 & 0.318461 & 0.244684 & 0.215898 & 0.193172 \\ 
                               & Total   & 0.347830 & 0.614727 & 0.463773 & 0.409212 & 0.366137 \\ 
                               \hline 
  \multirow{3}{*}{joint\_cvar} & C0 Risk & 0.132718 & 0.242794 & 0.176957 & 0.156138 & 0.139703 \\ 
                               & C1 Risk & 0.226791 & 0.383421 & 0.302387 & 0.266812 & 0.238727 \\ 
                               & Total   & 0.359508 & 0.626215 & 0.479344 & 0.422951 & 0.378430 \\ 
  \tabucline[1pt]{-}
  \multirow{3}{*}{asym\_risk}  & C0 Risk & 0.251054 & 0.431147 & 0.334738 & 0.295357 & 0.264267 \\ 
                               & C1 Risk & 0.092539 & 0.169554 & 0.123385 & 0.108869 & 0.097409 \\ 
                               & Total   & 0.343593 & 0.600701 & 0.458124 & 0.404227 & 0.361676 \\ 
                               \hline 
  \multirow{3}{*}{one\_cvar}   & C0 Risk & 0.167050 & 0.296830 & 0.222733 & 0.196529 & 0.175842 \\ 
                               & C1 Risk & 0.128815 & 0.229708 & 0.171754 & 0.151547 & 0.135595 \\ 
                               & Total   & 0.295865 & 0.526538 & 0.394487 & 0.348077 & 0.311437 \\ 
                               \hline
  \multirow{3}{*}{risk\_cvar}  & C0 Risk & 0.168882 & 0.299515 & 0.225176 & 0.198685 & 0.177771 \\ 
                               & C1 Risk & 0.123088 & 0.220300 & 0.164118 & 0.144810 & 0.129567 \\ 
                               & Total   & 0.291970 & 0.519815 & 0.389294 & 0.343495 & 0.307337 \\ 
                               \hline
  \multirow{3}{*}{two\_risk}   & C0 Risk & 0.152290 & 0.269093 & 0.203053 & 0.179165 & 0.160305 \\ 
                               & C1 Risk & 0.110126 & 0.195772 & 0.146835 & 0.129560 & 0.115922 \\ 
                               & Total   & 0.262416 & 0.464865 & 0.349888 & 0.308725 & 0.276227 \\ 
                               \hline
  \multirow{3}{*}{two\_cvar}   & C0 Risk & 0.240685 & 0.415233 & 0.320913 & 0.283158 & 0.253352 \\ 
                               & C1 Risk & 0.103057 & 0.188842 & 0.137409 & 0.121244 & 0.108481 \\ 
                               & Total   & 0.343742 & 0.604075 & 0.458322 & 0.404402 & 0.361833 \\
  \tabucline[2pt]{-}
\end{tabu}
\caption{Risk Evalutation for the ``pima-indians-diabetes'' data set -- Displaying the expectation of error, Mean Semi-deviation, and Avarage Value at Risk for the $\alpha$ quantiles 0.75, 0.85, and  0.95}
\label{tbl:risk_eval-pima-indians-diabetes}
\end{table*}

Table \ref{tbl:main-results_seismic-bumps} contains the risk functional evalutation for the ``pima-indians-data''.
It is interesting that the ``two\_risk'' model has the lowest total risk with respect to every risk functional, despite the fact that is not the best performing model in terms of $F_1$-score or AUC.
This leads us to believe that there may be room for additional exploration with regard to performance metrics and evaluation.


We continue with the performance evalution on the third and final dataset, whose main performance metrics are shown in Table \ref{tbl:main-results_seismic-bumps}.
One can immediately observe, that no model performs particularly well on this dataset.
We have chosen this data set for being particularly imbalanced and containing categorical varibles.

\begin{table*}[h!]
\centering
\small
  \begin{tabu}{rrrrrrrr}
  \multicolumn{8}{l}{$F_1$-score Optimized Classifiers} \\
  \hline
  \multicolumn{2}{r}{exp\_val} & joint\_cvar & asym\_risk & one\_cvar & risk\_cvar & two\_risk & two\_cvar \\
  \hline
  lambda &  &  & 0.61 & 0.60 & 0.59 & 0.53 & 0.70 \\ 
  alpha\_1 &  &  &  &  &  &  & 0.92 \\ 
  alpha\_2 &  & 0.60 &  & 0.86 & 0.84 &  & 0.92 \\ 
  C0 Errors & 471 & 203 & 269 & 248 & 230 & 270 & 201 \\ 
  C1 Errors & 64 & 93 & 83 & 85 & 87 & 83 & 94 \\ 
  FPR & 0.19511 & 0.08409 & 0.11143 & 0.10273 & 0.09528 & 0.11185 & 0.08326 \\ 
  Recall & 0.62353 & 0.45294 & 0.51176 & 0.50000 & 0.48824 & 0.51176 & 0.44706 \\ 
  Precision & 0.18371 & 0.27500 & 0.24438 & 0.25526 & 0.26518 & 0.24370 & 0.27437 \\ 
  $F_1$-score & 0.28380 & 0.34222 & 0.33080 & 0.33797 & \textbf{0.34369} & 0.33017 & 0.34004 \\ 
  AUC & 0.76157 & 0.75482 & \textbf{0.76187} & 0.75595 & 0.75496 & 0.75133 & 0.75629 \\ 
  \hline
  \multicolumn{8}{l}{} \\
  \multicolumn{8}{l}{AUC Optimized Classifiers} \\
  \hline
  \multicolumn{2}{r}{exp\_val} & joint\_cvar & asym\_risk & one\_cvar & risk\_cvar & two\_risk & two\_cvar \\
  \hline
  lambda &  &  & 0.60 & 0.47 & 0.47 & 0.49 & 0.47 \\ 
  alpha\_1 &  &  &  &  &  &  & 0.56 \\ 
  alpha\_2 &  & 0.93 &  & 0.75 & 0.58 &  & 0.56 \\
  C0 Errors & 471 & 261 & 292 & 812 & 817 & 571 & 633 \\ 
  C1 Errors & 64 & 84 & 82 & 50 & 48 & 62 & 54 \\ 
  FPR & 0.19511 & 0.10812 & 0.12096 & 0.33637 & 0.33844 & 0.23654 & 0.26222 \\ 
  Recall & 0.62353 & 0.50588 & 0.51765 & 0.70588 & 0.71765 & 0.63529 & 0.68235 \\ 
  Precision & 0.18371 & 0.24784 & 0.23158 & 0.12876 & 0.12993 & 0.15906 & 0.15487 \\ 
  $F_1$-score & 0.28380 & \textbf{0.33269} & 0.32000 & 0.21779 & 0.22002 & 0.25442 & 0.25245 \\ 
  AUC & 0.76157 & 0.76068 & 0.76360 & 0.76489 & 0.76611 & 0.76344 & \textbf{0.76637} \\ 
  \hline
  \end{tabu}
  \caption{Main results table for the ``seismic-bumps'' dataset -- Displaying the model parameters for the each model formulation as well as the corresponding performance metrics.}
  \label{tbl:main-results_seismic-bumps}
\end{table*}

Again, we see the ``risk\_cvar'' formulation as having the best $F_1$-score, followed very closely by the ``joint\_cvar'' formulation. In terms of AUC, it is the ``two\_cvar'' formulation that leads group, but again at a significant cost of the $F_1$-score. 
Looking at Figure \ref{fig:best-models-seismic-bumps-roc}, we can see room for improvemnts to the this by changing the threshold on the AUC prioritzed ``two\_cvar'' model.
We observe that in terms of stability to that respect, the ``asy\_risk'' formulation along with ``joint\_cvar'' benchmark have less variation.

\begin{figure*}[h!]
\begin{center}
\includegraphics[width=0.95\linewidth]{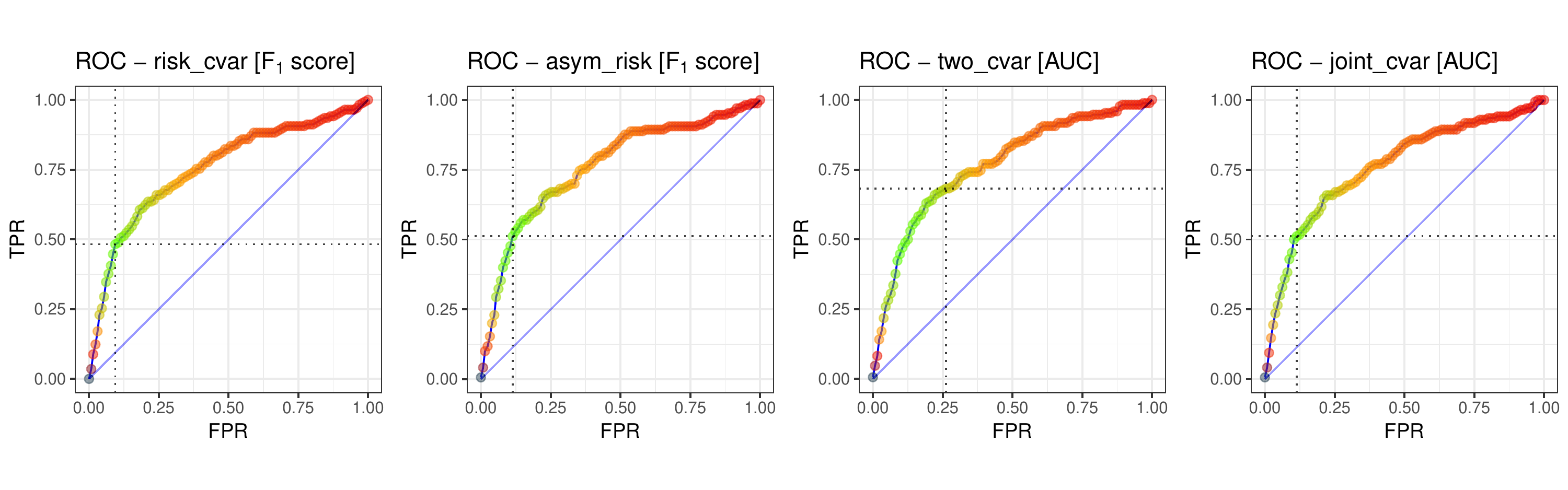}
\caption{ROC plots for the best performing model formulations on the ``seismic-bumps'' data: ``risk\_cvar'' with the best $F_1$-score, ``one\_cvar'', ``joint\_cvar'', ``two\_cvar'' formulation featuring the best AUC value}
\label{fig:best-models-seismic-bumps-roc}
\end{center}
\end{figure*}

Turning the attention to the risk functional evaluation in Table \ref{tbl:risk_eval-seismic-bumps}, we observe that the ``exp\_val'' benchmark model has the lowest total on the ``seismic-bumps''. However, being that this dataset is very imbalanced, we can see how significantly different the risk functional evaluation is between the two classes for each model formulation.

\begin{table*}[h!]
\centering
\small
\begin{tabu}{rr|rrrrr}
  \multicolumn{7}{l}{seismic-bumps} \\
  \tabucline[2pt]{-}
   \multicolumn{2}{r}{} & Expectation & MSD & AVaR$_{0.75}$ & AVaR$_{0.85}$ & AVaR$_{0.95}$ \\
  \hline
  \multirow{3}{*}{exp\_val}    & C0 Risk & 0.039589 & 0.072043 & 0.052786 & 0.046576 & 0.041673 \\ 
                               & C1 Risk & 0.064462 & 0.106979 & 0.085950 & 0.075838 & 0.067855 \\ 
                               & Total   & 0.104052 & 0.179022 & 0.138735 & 0.122414 & 0.109528 \\ 
                               \hline 
  \multirow{3}{*}{joint\_cvar} & C0 Risk & 0.015641 & 0.030007 & 0.020854 & 0.018401 & 0.016464 \\ 
                               & C1 Risk & 0.131682 & 0.199685 & 0.175576 & 0.154920 & 0.138613 \\ 
                               & Total   & 0.147323 & 0.229693 & 0.196430 & 0.173321 & 0.155077 \\ 
  \tabucline[1pt]{-}
  \multirow{3}{*}{asym\_risk}  & C0 Risk & 0.018930 & 0.035855 & 0.025239 & 0.022270 & 0.019926 \\ 
                               & C1 Risk & 0.099935 & 0.156758 & 0.133246 & 0.117570 & 0.105194 \\ 
                               & Total   & 0.118864 & 0.192613 & 0.158485 & 0.139840 & 0.125120 \\ 
                               \hline 
  \multirow{3}{*}{one\_cvar}   & C0 Risk & 0.019387 & 0.036922 & 0.025850 & 0.022809 & 0.020408 \\ 
                               & C1 Risk & 0.116238 & 0.179858 & 0.154983 & 0.136750 & 0.122355 \\ 
                               & Total   & 0.135625 & 0.216780 & 0.180833 & 0.159559 & 0.142763 \\ 
                               \hline
  \multirow{3}{*}{risk\_cvar}  & C0 Risk & 0.015942 & 0.030445 & 0.021256 & 0.018755 & 0.016781 \\ 
                               & C1 Risk & 0.107669 & 0.164839 & 0.143559 & 0.126669 & 0.113336 \\ 
                               & Total   & 0.123611 & 0.195284 & 0.164814 & 0.145424 & 0.130116 \\ 
                               \hline
  \multirow{3}{*}{two\_risk}   & C0 Risk & 0.015797 & 0.029943 & 0.021062 & 0.018584 & 0.016628 \\ 
                               & C1 Risk & 0.088633 & 0.139315 & 0.118177 & 0.104274 & 0.093298 \\ 
                               & Total   & 0.104430 & 0.169258 & 0.139239 & 0.122858 & 0.109926 \\ 
                               \hline
  \multirow{3}{*}{two\_cvar}   & C0 Risk & 0.013332 & 0.025589 & 0.017776 & 0.015685 & 0.014034 \\ 
                               & C1 Risk & 0.110821 & 0.167536 & 0.147762 & 0.130378 & 0.116654 \\ 
                               & Total   & 0.124153 & 0.193126 & 0.165538 & 0.146063 & 0.130688 \\ 
  \tabucline[2pt]{-}
\end{tabu}
\caption{Risk Evalutation for the ``seismic-bumps'' data set -- Displaying the expectation of error, Mean Semi-deviation, and Avarage Value at Risk for the $\alpha$ quantiles 0.75, 0.85, and  0.95}
\label{tbl:risk_eval-seismic-bumps}
\end{table*}

Notice, in Figure \ref{fig:seismic-bumps_cdf-roc_fscore}, how the ``exp\_val'' benchmark stands alone compared to the well grouped risk aware models, which includes the benchmark formulation ``joint\_cvar''. Similarly, as on the previous dataset, the ROC curves are very much grouped.

In summary, the $F_1$-score prioritized model consistently provides small but significant improvement over the baseline models.


\begin{figure*}[!h]
\centering
\includegraphics[width=0.545\linewidth]{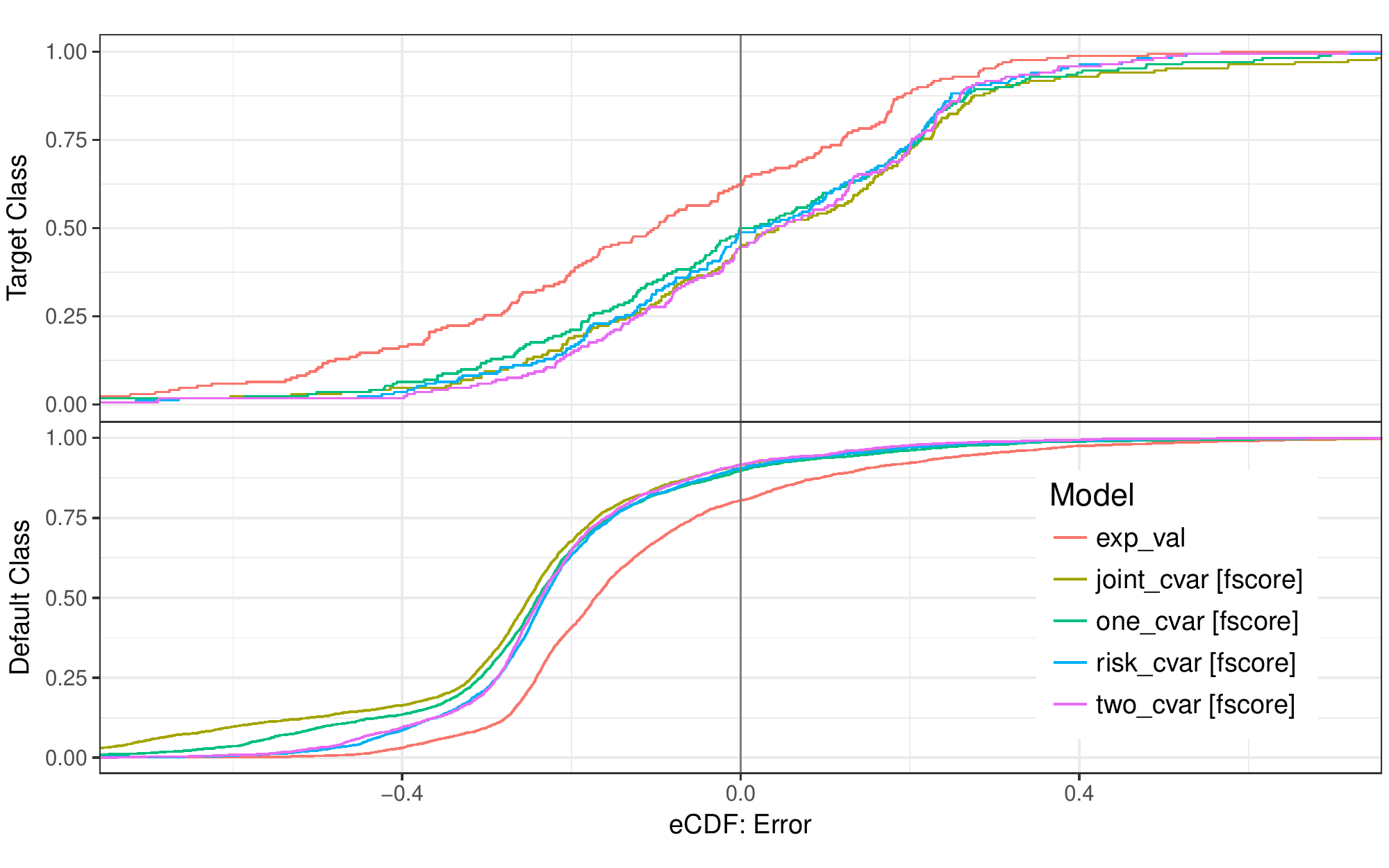}
\includegraphics[width=0.445\linewidth]{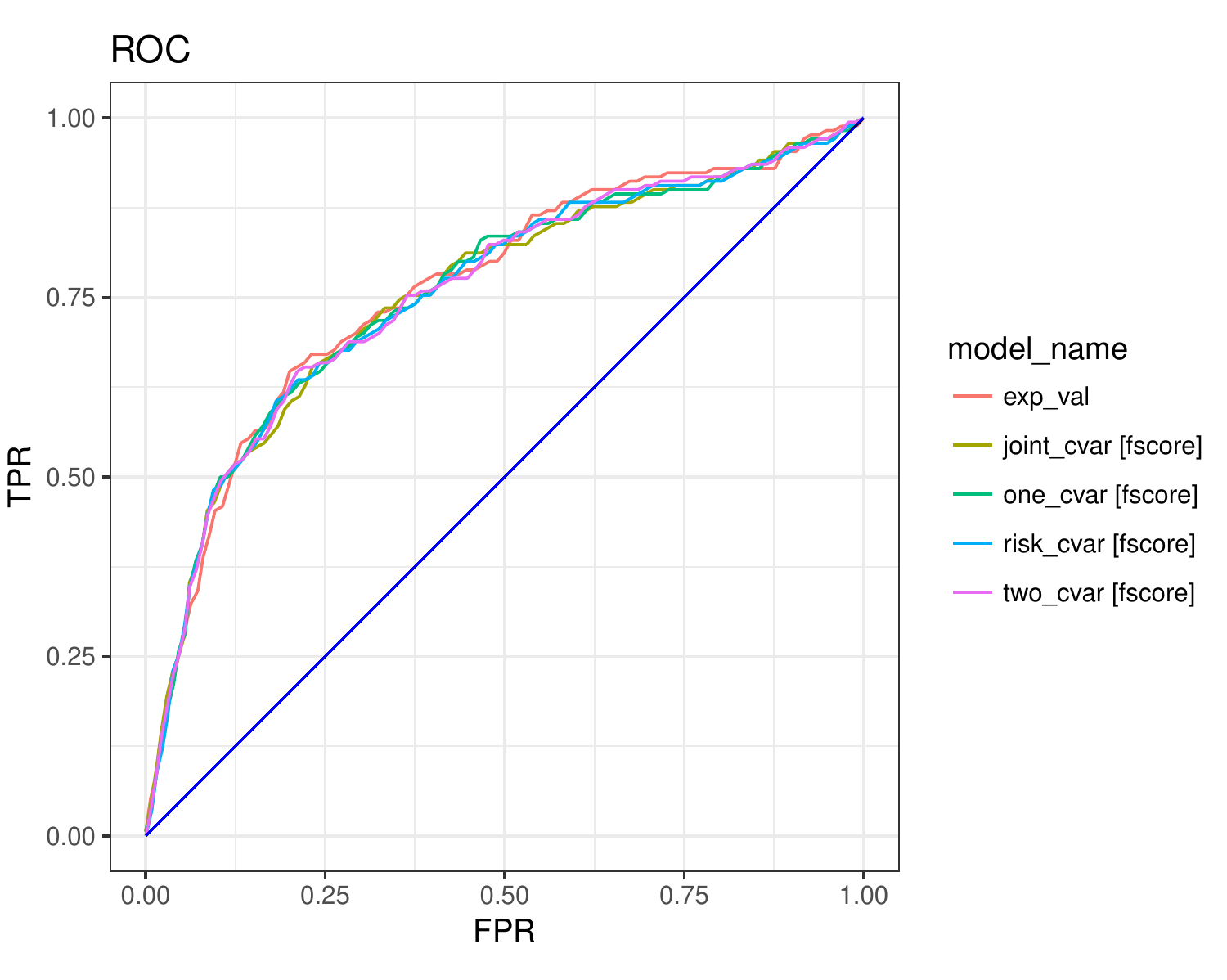}
\caption{Empirical distribution of error realizations comparing risk-averse loss function formulations to benchmarks [$F_1$-score] on the ``seismic-bumps'' dataset (left) and the corresponding ROC curves (right)}
\label{fig:seismic-bumps_cdf-roc_fscore}
\end{figure*}


\begin{figure*}[h!]
\centering
\includegraphics[width=0.325\linewidth]{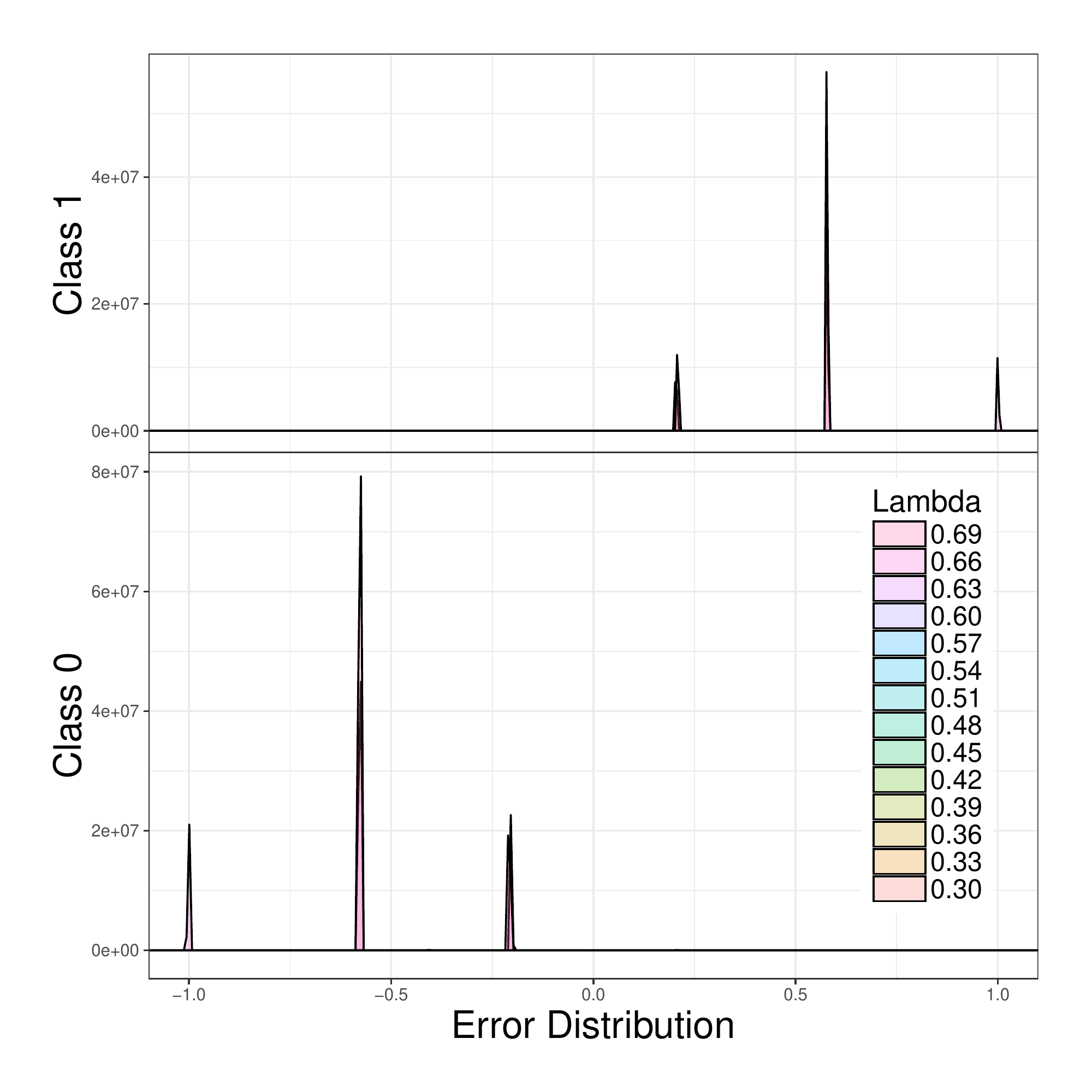}
\includegraphics[width=0.325\linewidth]{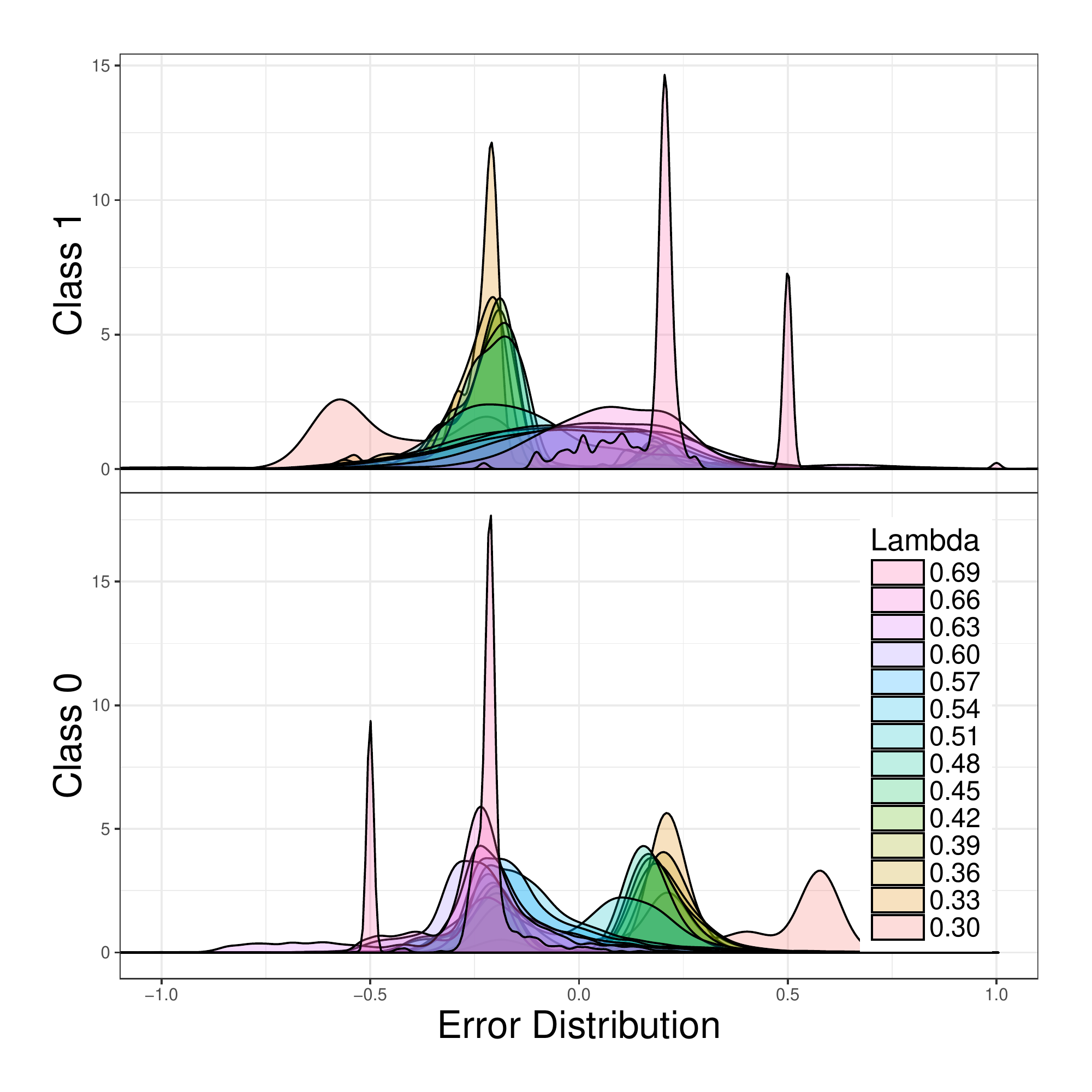}
\includegraphics[width=0.325\linewidth]{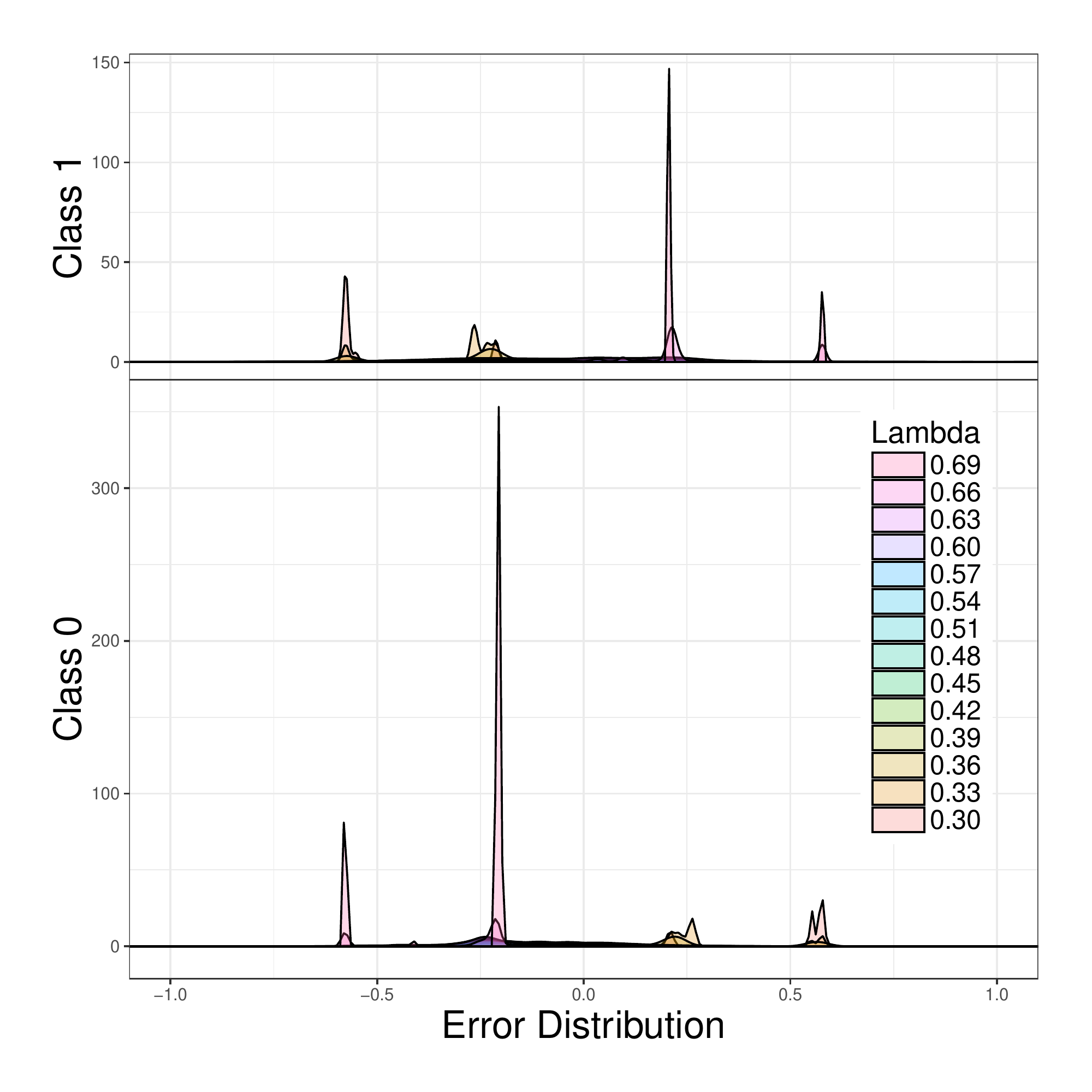}
\includegraphics[width=0.325\linewidth]{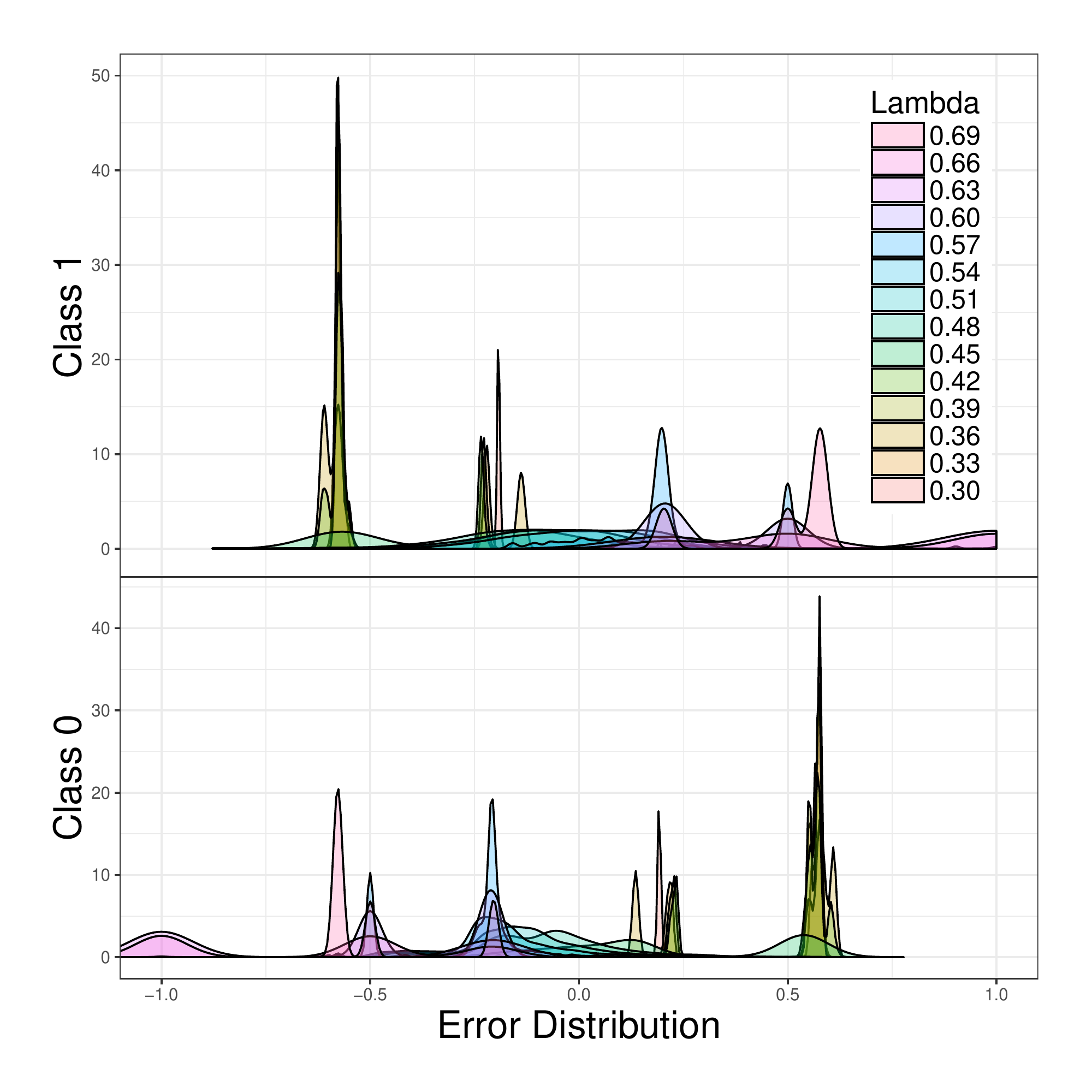}
\includegraphics[width=0.325\linewidth]{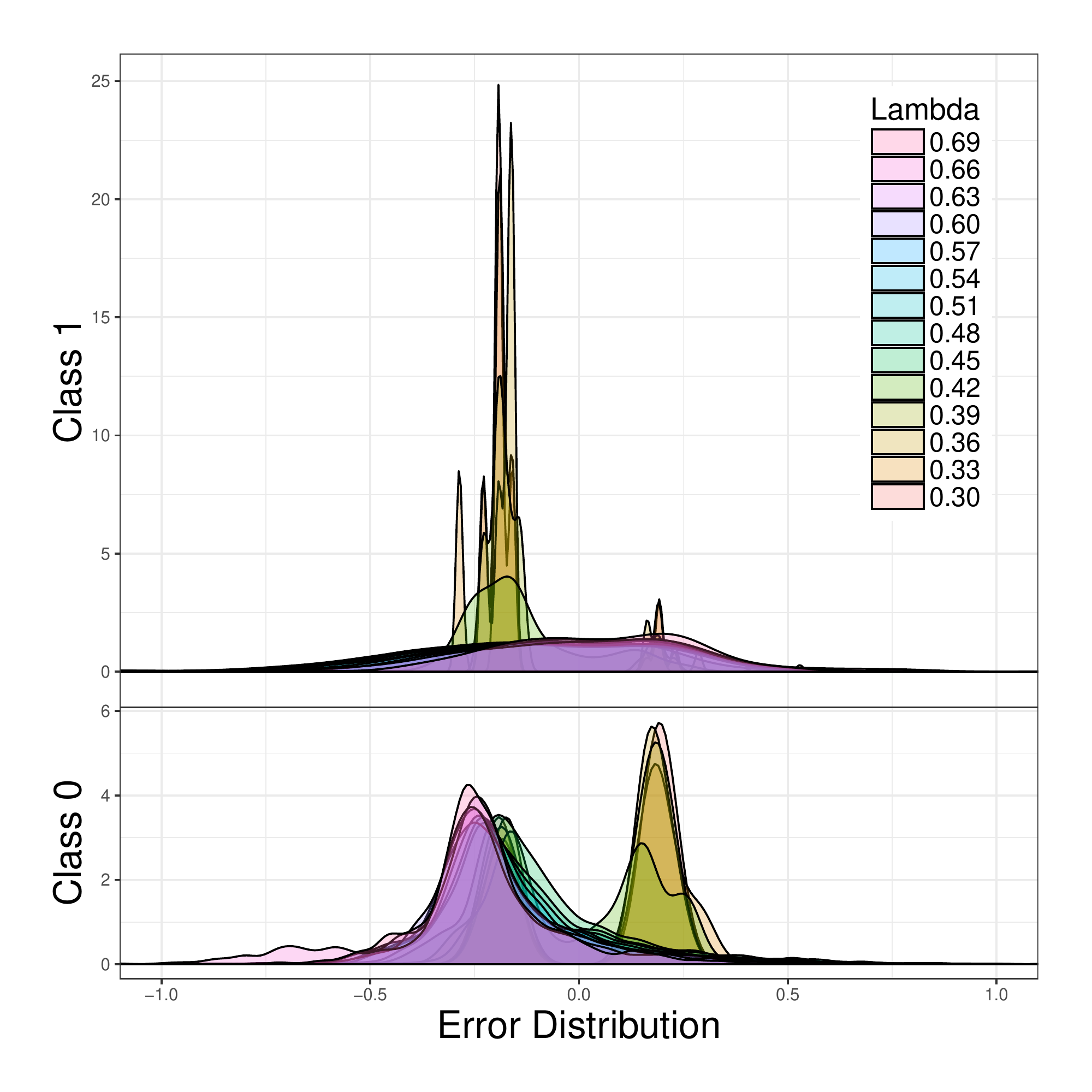}
\caption{The distribution of error displayed as smoothed histogram for each of five proposed formulations for the risk-averse SVM problem e.g. ``asym\_risk'', ``one\_cvar'', ``risk\_cvar'', ``two\_risk'', and ``two\_cvar'' all
using the same set of $\lambda$ values, with other parameters fixed, on the ``seismic-bumps'' dataset}
\label{fig:several_traverse1}
\end{figure*}


\subsection{Flexibility} 
\label{sub:flexibility}

Our approach provides additional flexibility which is generally not available for classification methods like soft-margin SVM. 
We allow the user to implement a predetermined attitude toward risk of misclassification, and to explore the Pareto-efficient frontier of classifiers.
The efficient frontier can be used to chose a risk-averse classifier according additional criterion as the F$_1$--score, AUC, or other similar performance metrics, as discussed in the previsou section.

We traverse the Pareto frontier by varying $\lambda$ from $0.4$ to $0.7$ and observe that the solution is rather sensitive to the scalarization used in the loss function.
In Figures \ref{fig:several_traverse1} 
, we show the resulting error densities from such a traversal.
We can observe how varying the weight between the two risk measures allows us to obtain a family of risk-averse Pareto-optimal classifiers.


The Pareto frontier looks substantially different when different combinations of risk measures are used. 
Further research would reveal the effect of higher order risk measures and their ability to create a classifier with highly discriminant powers.
We have chosen the probability level for the Average Value-at-Risk in a similar way.
We observe that the loss function ``one\_cvar'' consistently provides the best performance. 
A close second, is the loss function ``risk\_cvar,'' which has a similar structure.
Interestingly, using the same risk measure on both classes does not perform as well.


\section{Concluding Remarks} 
\label{sec:concluding_remarks}

This paper proposes a novel approach to classification problems by leveraging mathematical models of risk.
We have formulated several optimization problems for optimizing a classifier over a parametric family of functions. 
The problem's objective is a weighted sum of risk-measures, associated with the classification error of the classes: each class may be treated with an individual risk preference. 
We have shown the existence of an optimal risk-sharing classifier under mild assumptions. 
Additionally, we demonstrate that the optimal risk-sharing classifier also solves an implicit risk-neutral classification problem, in which the empirical probabilities of the data points are replaced by a probability distribution from the subdifferential of the risk-measures.   
We have provided a more specific problem formulation for the case of binary classification and have conducted experiments on three data sets.
Further, we have compared our approach to three benchmarks, which use the minimization of the total expected error, the Huber function, and the Average Value-at-Risk as presented in \cite{gotoh2017support}.
Our observations are the following. On the data sets for which traditional formulations perform well, the novel approach performs on par or slightly better depending on the particular choice of risk measures and parameters.
The proposed approach has an advantage on all data sets as measured by the F$_1$-score.
Exploring the Pareto-efficient frontier provides additional flexibility and is a tool for customizing the classifier.
As we see from the numerical results, we achieve larger recall or precision by adjusting the scalarization factor $\lambda$.
Overall, this is an extremely flexible approach which allows fine-tuning leading allowing the user to achieve the best possible result in the chosen metric.


\bibliographystyle{plain}

\end{document}